\documentclass{article} 
\usepackage{times}
\usepackage{iclr2025_conference}

\usepackage{bm}
\usepackage{amssymb}
\usepackage{amsfonts}       
\usepackage{amsmath}
\usepackage{amssymb}
\usepackage{mathtools}
\mathtoolsset{showonlyrefs}
\usepackage{mathtools}

\usepackage{physics}
\usepackage{mathrsfs}

\newcommand{\explain}[1]{\tag*{(#1)}}

\newcommand{\fS}{\mathcal{S}}
\newcommand{\fA}{\mathcal{A}}
\newcommand{\fY}{\mathcal{Y}}

\newcommand{\fF}{\mathcal{F}}
\newcommand{\fO}{\mathcal{O}}

\newcommand{\R}{\mathbb{R}}
\newcommand{\E}{\mathbb{E}}

\newcommand{\ns}{{|\fS|}}

\newcommand{\bop}{\mathcal{T}}

\DeclarePairedDelimiter\floor{\lfloor}{\rfloor}

\newcommand{\atd}{{\text{$A^\top$TD}}}
\newcommand{\attd}{{\text{$A^\top_t$TD}}}

\newcounter{assu_counter}
\numberwithin{assu_counter}{section}
\newtheorem{assumption}[assu_counter]{Assumption}
 
\newtheorem{theorem}{Theorem}
\newtheorem{lemma}[theorem]{Lemma}

\newcommand{\BlackBox}{\rule{1.5ex}{1.5ex}}  
\ifdefined\proof
    \renewenvironment{proof}{\par\noindent{\bf Proof\ }}{\hfill\BlackBox\\[2mm]}
\else
    \newenvironment{proof}{\par\noindent{\bf Proof\ }}{\hfill\BlackBox\\[2mm]}
\fi

\usepackage[hidelinks,colorlinks=true,linkcolor=blue,citecolor=blue]{hyperref}    
\usepackage{url}

\title{Revisiting a Design Choice in \\ Gradient Temporal Difference Learning}

\author{Xiaochi Qian \\
Department of Computer Science\\
University of Oxford\\
\texttt{xiaochi.joe.qian@gmail.com} \\
\And
Shangtong Zhang \\
Department of Computer Science \\
University of Virginia \\
\texttt{shangtong@virginia.edu} \\
}

\iclrfinalcopy 
\begin{document}

\maketitle

\begin{abstract}
Off-policy learning enables a reinforcement learning (RL) agent to reason counterfactually about policies that are not executed and is one of the most important ideas in RL. It, however, can lead to instability when combined with function approximation and bootstrapping, two arguably indispensable ingredients for large-scale reinforcement learning. This is the notorious deadly triad. The seminal work \citet{sutton2009convergent} pioneers Gradient Temporal Difference learning (GTD) as the first solution to the deadly triad, which has enjoyed massive success thereafter. During the derivation of GTD, some intermediate algorithm, called $A^\top$TD, was invented but soon deemed inferior. In this paper, we revisit this $A^\top$TD and prove that a variant of $A^\top$TD, called $A_t^\top$TD, is also an effective solution to the deadly triad. Furthermore, this $A_t^\top$TD only needs one set of parameters and one learning rate. By contrast, GTD has two sets of parameters and two learning rates, making it hard to tune in practice. We provide asymptotic analysis for $A^\top_t$TD and finite sample analysis for a variant of $A^\top_t$TD that additionally involves a projection operator. The convergence rate of this variant is on par with the canonical on-policy temporal difference learning.
\end{abstract}

\section{Introduction}
\label{sec intro}

Off-policy learning \citep{watkins1989learning,precup:2000:eto:645529.658134,maei2011gradient,sutton2016emphatic,li2019perspective} is arguably one of the most important ideas in 
reinforcement learning (RL, \citet{sutton2018reinforcement}).
Different from on-policy learning \citep{sutton1988learning},
where an RL agent learns quantities of interest of a policy by executing the policy itself,
an off-policy RL agent learns quantities of interest of a policy by executing a different policy.
This flexibility offers additional safety \citep{dulac2019challenges} and data efficiency \citep{lin1992self,sutton2011horde}.

Off-policy learning, however, can lead to instability if combined with function approximation and bootstrapping,
two other arguably indispensable techniques for any large-scale RL applications.
The idea of function approximation \citep{sutton1988learning} is to represent quantities of interest with parameterized functions instead of look-up tables.
The idea of bootstrapping \citep{sutton1988learning} is to construct update targets for an estimator by using the estimator itself recursively.
This instability resulting from off-policy learning, function approximation, and bootstrapping is called the deadly triad \citep{baird1995residual,sutton2018reinforcement,zhang2022thesis}.

The seminal work \citet{sutton2009convergent} pioneers the first solution to the deadly triad,
called 
Gradient Temporal Difference learning (GTD).
Thereafter,
GTD has been massively studied and enjoyed celebrated success \citep{sutton2009convergent,sutton2009fast,bhatnagar2009convergent,maei2010gq,maei2010toward,maei2011gradient,mahadevan2014proximal,liu2015finite,du2017stochastic,wang2017finite,yu2017convergence,xu2019two,wang2020finite,wai2020provably,ghiassian2020gradient,zhang2020average}.
During the derivation of GTD in \citet{sutton2009convergent},
an intermediate algorithm called $\atd$ was invented but soon deemed inferior.
In \citet{sutton2009convergent}, it is said that ``although we find this algorithm interesting, we do not consider it further here because it requires $\fO(K^2)$ memory and computation per time step''. 
Here, $K$ refers to the feature dimension in linear function approximation.
In this paper, 
we propose a variant of $\atd$, called $\attd$,
which has $\fO(K)$ computation per step, and the memory cost is $\fO(K\ln^2 t)$.
Here, $t$ refers to the time step.
Admittedly, $\ln^2 t$ diverges to $\infty$ eventually.
However, we argue that this memory cost is negligible in any empirical implementations.
For example,
our universe has an age of around 14 billion years.
Consider a modern 3 GHz CPU.
Suppose that an RL agent runs 1 step every CPU clock and starts from the very beginning of our universe.
Then until now it roughly has run $T = 14 \times 10^9 \times 3.1536 \times 10^7 \times 3 \times 10^9 \approx 10^{27}$ steps.
We then have $\ln^2 T \approx 4000$.
In light of this,
we claim that $\attd$ does not have any real drawback in terms of memory compared with GTD.
$\attd$, however,
has only one set of parameters and one learning rate.
By contrast,
GTD has two sets of parameters and two learning rates,
making it hard to tune in practice \citep{sutton2009convergent}.
We prove that $\attd$ eventually converges to the same solution as GTD.
We also demonstrate that if an additional projection operator is used,
$\attd$ also enjoys the same convergence rate as the canonical on-policy TD.
The assumptions in our analysis are all standard.

\section{Background}
In this paper,
all vectors are columns.
We use $\norm{\cdot}$ to denote the $\ell_2$ vector and matrix norm. 
We use functions and vectors interchangeably when it does not confuse.
For example, if $f$ is a function from $\fS$ to $\R$,
we also use $f$ to denote a vector in $\R^\ns$,
whose $s$-th element is $f(s)$.

We consider an infinite horizon Markov Decision Process (MDP, \citet{puterman2014markov}) with a finite state space $\fS$,
a finite action space $\fA$,
a reward function $r: \fS \times \fA \to \R$,
a transition function $p: \fS \times \fS \times \fA \to [0, 1]$,
and a discount factor $\gamma \in [0, 1)$.
At time step 0,
a state $S_0$ is sampled from some initial distribution $p_0$.
At time step $t$,
an agent at a state $S_t$ takes an action $A_t \sim \pi(\cdot | S_t)$.
Here $\pi: \fA \times \fS \to [0, 1]$ is the policy being followed.
A reward $R_{t+1} \doteq r(S_t, A_t)$ is then emitted, and a successor state $S_{t+1}$ is sampled from $p(\cdot | S_t, A_t)$.

The return at time step $t$ is defined as
    $G_t \doteq \sum_{i=0}^{\infty} \gamma^i R_{t+i+1}$,
which allows us to define the state value function as
    $v_{\pi}(s) \doteq \E_{\pi, p}\left[G_t | S_t = s\right]$.
The value function $v_\pi$ is the unique fixed point of the Bellman operator
    $\bop_\pi v \doteq r_\pi + \gamma P_\pi v$.
Here $r_\pi \in \R^{\ns}$ is the reward vector induced by $\pi$, defined as $r_\pi(s) \doteq \sum_a \pi(a|s) r(s, a)$.
$P_\pi \in \R^{\ns \times \ns}$ is the transition matrix induced by $\pi$,
i.e.,
$P_\pi(s, s') \doteq \sum_a \pi(a|s)p(s'|s, a)$.

Estimating $v_\pi$ is one of the most important tasks in RL and is called policy evaluation. 
Linear function approximation is commonly used for policy evaluation \citep{sutton1988learning}.
Consider a feature function $x: \fS \to \R^K$ that maps a state $s$ to a $K$-dimensional feature $x(s)$.
We then use $x(s)^\top w$ to approximate $v_\pi(s)$.
Here $w \in \R^K$ is the learnable weight.
Let $X \in \R^{\ns \times K}$ be the feature matrix,
whose $s$-th row is $x(s)^\top$.
The goal is then to adapt $w$ such that $Xw \approx v_\pi$.
Linear TD \citep{sutton1988learning} updates $w$ iteratively as 
\begin{align}
    \label{eq linear td}
    w_{t+1} \doteq w_t + \alpha_t \left(R_{t+1} + \gamma x_{t+1}^\top w_t - x_t^\top w_t\right)x_t.
\end{align}
Here, we use $x_t \doteq x(S_t)$ as shorthand.
Under mild conditions,
the iterates $\qty{w_t}$ in~\eqref{eq linear td} converge almost surely \citep{tsitsiklis1997analysis}.

It is commonly the case that we want to estimate $v_\pi$ without actually sampling actions from $\pi$
due to various concerns, e.g., 
safety \citep{dulac2019challenges}, data efficiency \citep{lin1992self,sutton2011horde}.
Off-policy learning makes this possible.
In particular, 
instead of sampling $A_t$ according $\pi(\cdot | S_t)$,
off-policy method samples $A_t$ according to another policy $\mu$.
Here,
the policy $\pi$ is called the target policy and the policy $\mu$ is called the behavior policy.
\emph{For the rest of the paper,
we always consider the off-policy setting},
i.e.,
\begin{align}
    \label{eq data stream}
    A_t \sim \mu(\cdot | S_t), R_{t+1} = r(S_t, A_t), S_{t+1} \sim p(\cdot | S_t, A_t).
\end{align}
Since the behavior policy $\mu$ is different from the target policy $\pi$,
importance sampling ratio is used to account for this discrepancy,
which is defined as 
    $\rho(s, a) \doteq \frac{\pi(a|s)}{\mu(a|s)}$.
In particular, we use as shorthand $\rho_t \doteq \rho(S_t, A_t)$.
Off-policy linear TD then updates $w$ iteratively as
\begin{align}
    \label{eq linear off-policy td}
    w_{t+1} \doteq w_t + \alpha_t \rho_t \left(R_{t+1} + \gamma x_{t+1}^\top w_t - x_t^\top w_t\right)x_t.
\end{align}
It is well-known \citep{sutton2009convergent} that if off-policy linear TD converged,
it would converge to a $w_*$ satisfying
\begin{align}
    \label{eq td fixed point}
    Aw_* + b = 0,
\end{align}
where
\begin{align}
    \label{eq A matrix}
    A \doteq&  X^\top D_\mu(\gamma P_\pi - I)X, \,
    b \doteq  X^\top D_\mu r_\pi.
\end{align}
Here, $d_\mu$ is the stationary distribution of the Markov chain induced by the behavior policy $\mu$,
and $D_\mu$ is a diagonal matrix with the diagonal being $d_\mu$.
Unfortunately,
the possible divergence of off-policy linear TD in~\eqref{eq linear off-policy td} is well documented \citep{baird1995residual,sutton2016emphatic,sutton2018reinforcement}.
This divergence exercises the deadly triad.


Instead of using off-policy linear TD in~\eqref{eq linear off-policy td} to find $w_*$,
one natural approach for policy evaluation in the off-policy setting is then to solve 
    $Aw+b = 0$
directly,
probably with stochastic gradient descent on the objective
    $L(w) \doteq \norm{Aw + b}^2$.
The on-policy version of this objective (i.e., with $\mu = \pi$) is first introduced in \citet{yao2008preconditioned} to derive preconditioned TD.
The off-policy version considered in this paper
is first used by \citet{sutton2009convergent} to derive GTD, and this objective is called \emph{the norm of the expected TD update} (NEU) in \citet{sutton2009fast}.
The gradient of $L(w)$ can be easily computed as
    $\nabla L(w) = 2A^\top (A w + b)$.
One can, therefore, update $w$ as 
\begin{align}
    \label{eq gradient of obj}
    w_{t+1} \doteq w_t - \alpha_t A^\top (A w_t + b).
\end{align}
Since we do not know $A$ and $b$,
we need to estimate $A^\top (A w_t + b)$ with samples.
The idea of $\atd$ in \citet{sutton2009convergent} is to estimate $A^\top$ as
\begin{align}
    \textstyle A^\top \approx \frac{1}{t+1} \sum_{i=0}^{t} \rho_i (\gamma x_{i+1} - x_i) x_i^\top
\end{align}
and to estimate $Aw_t + b$ as
\begin{align}
    Aw_t + b \approx \rho_t \qty(R_{t+1} + \gamma x_{t+1}^\top w_t - x_t^\top w_t) x_t.
\end{align}
As said in \citet{sutton2009convergent},
$\atd$ is ``essentially conventional TD(0) prefixed by an estimate of the matrix $A^\top$''.
Apparently, computing and store this estimate of $A^\top$ requires $\fO(K^2)$ computation and memory per step, 
if we use a moving average implementation.
And it is unclear whether this $\atd$ is convergent.
Having deemed this $\atd$ inferior,
\citet{sutton2009convergent} rewrite the gradient as
\begin{align}
    \nabla L(w) = A^\top (A w + b) = A^\top X^\top D_\mu \left(\bop_\pi(Xw) - Xw\right)
\end{align}
and use a secondary weight $\nu \in \R^K$ to approximate $X^\top D_\mu \left(\bop_\pi(Xw) - Xw\right)$,
yielding the following well-known GTD algorithm
\begin{align}
    \delta_t \doteq& R_{t+1} + \gamma x_{t+1}^\top w_t - x_t^\top w_t, \\
    \nu_{t+1} \doteq& \nu_t + \alpha_t \left(\rho_t \delta_t x_t - \nu_t \right), \\
    w_{t+1} \doteq& w_t + \alpha_t \rho_t (x_t - \gamma x_{t+1}) x^\top_t \nu_t.
    \label{eq GTD} \tag{GTD}
\end{align}
The convergence and finite sample analysis of GTD is well established \citep{sutton2009convergent,sutton2009fast,liu2015finite,wang2017finite,yu2017convergence}.

\section{$\attd$: Revisiting the Design Choice of $\atd$}
In this paper,
we refine the idea of 
$\atd$ via estimating $A^\top$ with a single sample at time $t + f(t)$ as 
\begin{align}
    A^\top \approx \rho_{t+f(t)} \left(x_{t+f(t)} - \gamma x_{t+f(t) + 1}\right) x^\top_{t+f(t)},
\end{align}
where $f: \mathbb{N} \to \mathbb{N}$ is an increasing \emph{gap function}.
This yields the following novel algorithm:
\begin{align}
    \delta_t \doteq& R_{t+1} + \gamma x_{t+1}^\top w_t - x_t^\top w_t, \\
    w_{t+1} \doteq& w_t + \alpha_t \rho_{t+f(t)} \left(x_{t+f(t)} - \gamma x_{t+f(t) + 1}\right) x^\top_{t+f(t)} \rho_t \delta_t x_t \tag{$\attd$}.
    \label{eq direct gtd}
\end{align}
We call it $\attd$ to highlight that it uses a single sample to estimate $A^\top$.
In~\eqref{eq direct gtd}, 
the term $\rho_{t+f(t)} \rho_t \left(R_{t+1} + \gamma x_{t+1}^\top w_t - x_t^\top w_t\right)$ is a scalar,
the computational complexity of which is only $\fO\left(K\right)$.
If we compute the remaining term $\left(x_{t + f(t)} - \gamma x_{t + f(t) + 1}\right) x^\top_{t + f(t)} x_t$ from right to left,
the computational complexity is still $\fO\left(K\right)$.
In other words,
the computational complexity of~\eqref{eq direct gtd} is the same as~\eqref{eq GTD}.
The price we pay here is that we cannot start~\eqref{eq direct gtd} until the $\left(f(0)+1\right)$-th step, and we need to maintain a memory storing
\begin{align}
    \label{eq memory}
    x_t, \rho_t, x_{t+1}, \rho_{t+1}, \dots x_{f(t)}, \rho_{f(t)}, x_{f(t)+1}.
\end{align}
The size of this memory is $\fO\left(f(t)\right)$.
We will soon prove that the memory can be as small as $\Omega(\ln^2(t))$.
We argue that this memory overhead is negligible in any empirical implementations.
The gain is that we now do not need an additional weight vector,
making the algorithm easy to use.
We will have a few assumptions on the gap function $f$ shortly to facilitate the theoretical analysis.
But one example could simply be 
\begin{align}
f(t) = \floor{\ln^2(t+1)},
\end{align}
where $\floor{\cdot}$ is the floor function.
In other words, the choice of the gap function is simple and does not depend on any unknown problem structure.
To understand how this gap function works,
we consider a case where $f(t) = 0 \forall t$.
Then we are essentially estimating $A^\top$ and $Aw_t + b$ with the same sample $(x_t, R_{t+1}, x_{t+1})$.
We will then for sure run into the well-known double sampling issue\footnote{It is well-known that in general $\E[XY] \neq \E[X] \E[Y]$.}.
By using the gap function,
we use the sample at time $t+f(t)$ to estimate $A^\top$ and the sample at time $t$ to estimate $Aw_t + b$.
Despite that those two samples are still correlated due to the Markovian nature of the data stream~\eqref{eq data stream},
the increasing $f(t)$ gradually reduces the correlation.
The theoretical analysis in the following two sections confirms that such a simple gap function is enough to guarantee the desired convergence to $w_*$ with a desired convergence rate. 

We do note that throughout the paper we consider the canonical RL setting where only Markovian samples are available.
If instead i.i.d. samples are available,
addressing the aforementioned doubling sampling issue then becomes more straightforward -- one can simply use two independent samples to estimate $A^\top$ and $Aw_t + b$.
$\atd$ with i.i.d. samples are thoroughly investigated in \citet{yao2023new} and we refer the reader to \citet{yao2023new} for more details.

\section{Asymptotic Convergence Analysis of $\attd$}
\label{sec as convergence}
In this section,
we provide an asymptotic convergence analysis of~\eqref{eq direct gtd}.
The major technical challenge lies in the increasing gap function.
If $f(t)$ was a constant function, say $f(t) \equiv t_0$,
then one could start analyzing~\eqref{eq direct gtd} via constructing an augmented Markov chain with states 
    $Y_t \doteq \qty{S_t, A_t, \dots, S_{t+t_0}, A_{t+t_0}, S_{t+t_0+1}}$,
evolving in a finite space $(\fS \times \fA)^{t_0+1} \times \fS$.
Suppose the origin Markov chain $\qty{S_t}$ is ergodic,
this new chain $\qty{Y_t}$ would also be ergodic,
matching the ergodicity assumption of 
classical convergence results (e.g., Proposition 4.7 of \citet{bertsekas1996neuro}).\footnote{We note that when $f(t) \equiv t_0$, another key assumption does not hold. 
To see this, let $\hat A_t \doteq \rho_t x_t(\gamma x_{t+1} - x_t)^\top$. If $f(t) \equiv t_0$, the expected updates would be governed by $-\E\qty[\hat A_{t+t_0}^\top \hat A_t]$. 
However, this expectation is not equal to $-A^\top A$ due to the correlation between $\hat A_{t+t_0}$ and $\hat A_t$. 
So this expectation is unlikely to be negative definite.
But canonical results typically require this expectation to be negative define.
To summarize, if $f(t) \equiv t_0$,
the ergodicity assumption on the Markov chain in the canonical results can be fulfilled but the negative definiteness of the expected update matrix cannot be fulfilled.
}
When $f(t)$ is increasing,
the augmented chain, however, is now 
    $Y_t \doteq \qty{S_t, A_t, \dots, S_{t+f(t)}, A_{t+f(t)}, S_{t+f(t)+1}}$
which evolves in an infinite space
    $\bigcup_{i=1}^\infty (\fS \times \fA)^i \times \fS$.
Even if the original chain $\qty{S_t}$ is ergodic,
the new chain $\qty{Y_t}$ still behaves poorly in that it never visits the same augmented state twice.
This rules out the possibility of applying most,
if not all,
existing convergence results in the stochastic approximation community (e.g., \citet{DBLP:books/sp/BenvenisteMP90,kushner2003stochastic,borkar2009stochastic,liu2025ode}).
To proceed,
we instead use the skeleton iterates technique introduced by \citet{qian2024sureconvergenceratesconcentration}.
The key idea of this skeleton iterates technique is to divide the non-negative real axis into intervals of length $\qty{T_m}$ and examine the updates interval by interval.
Importantly,
we will require this $\qty{T_m}$ to diminish,
in a rate coordinated with the gap function $f(t)$ and the learning rate $\alpha_t$.

Besides the skeleton iterates technique, another important ingredient is the mixing of joint state distributions in Markov chains.
Consider a general Markov chain $\qty{Y_t}$.
Assume the chain is ergodic and let $d_\fY$ denote its invariant distribution.
Then the convergence theorem (see, e.g., \citet{levin2017markov}) yields
    $\lim_{t\to\infty}\Pr(Y_t = y) = d_\fY(y)$.
This convergence is uniform in $y$ and is geometrically fast.
Exploiting this convergence,
we are able to prove the convergence of joint state distributions, i.e.,
\begin{align}
    \lim_{t\to\infty}\Pr(Y_t = y, Y_{t+f(t)} = y') = d_\fY(y) d_\fY(y').
\end{align}
Intuitively,
this means the dependence between $Y_t$ and $Y_{t+f(t)}$ diminishes as $t$ goes to infinity (cf. Lemma 7.1 in \citet{vempala2005geometric}).
In~\eqref{eq direct gtd},
this means
the bias resulting from 
the correlation of the two samples at time $t+f(t)$ and time $t$ diminishes gradually.
Having introduced the two main technical ingredients in our analysis,
we are now ready to formally describe our main results.
We start with (standard) assumptions we make.

\begin{assumption}
    \label{assu chain}
    The Markov chain induced by the behavior policy $\mu$ is finite, irreducible, and aperiodic.
    And $\mu$ covers $\pi$, i.e., $\forall (s, a), \pi(a|s) > 0 \implies \mu(a|s) > 0$.
\end{assumption}
\begin{assumption}
    \label{assu feature}
    The feature matrix $X$ has a full column rank.
    The matrix $A$ defined in~\eqref{eq A matrix} is nonsingular.
\end{assumption}
Assumptions~\ref{assu chain} and~\ref{assu feature} are standard in the analysis of linear TD methods (see, e.g., \citet{tsitsiklis1997analysis,wang2017finite}).

\begin{assumption}
    \label{assu lr}
    The learning rates $\qty{\alpha_t}$ have the form of
        $\alpha_t = \frac{C_\alpha}{(t+1)^{\nu}}$,
    for some $\nu \in (\frac{2}{3}, 1]$.
\end{assumption}
Assumption~\ref{assu lr} considers learning rates of a specific form.
This is mostly for ease of presentation.

\begin{assumption}
    \label{assu gap}
    The gap function $f(t): \mathbb{N} \to \mathbb{N}$ is increasing and satisfies $\forall \chi \in [0, 1)$, \\
        $\textstyle \sum_{t=0}^\infty \chi^{f(t)} < \infty$.
    Moreover, there exist constants $\tau \in (0, \frac{3 \nu - 2}{2 \nu})$ and $C_\tau > 0$ such that $\forall t$,
      $f(t) \le C_\tau \alpha_t^{-\tau}$.  
\end{assumption}
Assumption~\ref{assu gap} is the most ``unnatural'' assumption we make and prescribes how the gap function should be chosen.
Intuitively,
those conditions prevent the gap function from growing too fast.
Despite seemingly complicated,
Lemma \ref{lem sufficient condition for gap assu} in the appendix confirms that simply setting
\begin{align}
    \label{eq gap function example}
    f(t) = \floor{h(t) \ln(t+1)}
\end{align}
with any non-negative increasing function $h(t)$ converging to $\infty$ as $t \to \infty$
fulfills the first condition of Assumption~\ref{assu gap}.
Here $\floor{x}$ is the floor function denoting the largest integer smaller than $x$.
A concrete example satisfying Assumption~\ref{assu lr} and~\ref{assu gap} is
\begin{align}
    \label{eq lr example}
    \nu =& 1, h(t) = \ln(t+1), f(t) = \floor{\ln^2(t+1)}, \tau=0.1.
\end{align}
We are now ready to present our main results.
\begin{theorem}
    \label{thm asym}
    Let Assumptions~\ref{assu chain},~\ref{assu feature},~\ref{assu lr}, \&~\ref{assu gap} hold.
    Then the iterates $\qty{w_t}$ generated by~\eqref{eq direct gtd} satisfies
    \begin{align}
        \lim_{t\to\infty} w_t = w_* \qq{a.s.,}
    \end{align}
    where $w_*$ is the TD fixed point defined in~\eqref{eq td fixed point}.
\end{theorem}

\begin{proof}
    Following \citet{qian2024sureconvergenceratesconcentration},
we define a sequence $\qty{T_m}_{m=0,1,\dots}$ as
\begin{align}
    \label{eq def big tm}
    \textstyle T_m = \frac{16 \max(C_\alpha,1)}{(\eta + 1)(m+1)^\eta},
\end{align}
where $C_\alpha$ is defined in Assumption~\ref{assu lr} and $\eta$ is some constant such that 
\begin{align}
    \label{eq def eta}
    \textstyle \frac{1}{2(1-\tau)} < \eta < \frac{\nu}{2-\nu}.
\end{align}
Here $\nu$ and $\tau$ are defined in Assumption~\ref{assu lr} and~\ref{assu gap} respectively.
Notably,
despite that we follow the skeleton iterates technique in \citet{qian2024sureconvergenceratesconcentration},
our analysis is more challenging than \citet{qian2024sureconvergenceratesconcentration} in that they only need to coordinate $\qty{T_m}$ with the learning rate $\alpha_t$ 
but we need to coordinate $\qty{T_m}$ with both the learning rate $\alpha_t$ and the gap function $f(t)$.
As a result,
\citet{qian2024sureconvergenceratesconcentration} only require $\eta \in (\frac{1}{2}, \frac{\nu}{2 - \nu})$ but we further require $\eta > \frac{1}{2(1 - \tau)}$,
which significantly complicates the analysis.

We now follow \citet{qian2024sureconvergenceratesconcentration} and divide the real line into intervals with approximate length $\qty{T_m}$.
To this end,
we define a sequence $\qty{t_m}$ as $t_0 \doteq 0$,
\begin{align}
    \label{eq def tm}
    \textstyle t_{m+1} \doteq \min\qty{k | \sum_{t=t_m}^{k-1} \alpha_t \geq T_m}, \, m=0,1,\dots
\end{align}
For simplicity,
define
\begin{align}
   \textstyle \bar \alpha_m \doteq \sum_{t=t_m}^{t_{m+1}-1} \alpha_t, \, m=0,1,\dots
\end{align}
Now, the real line has been divided into intervals of lengths $\qty{\bar \alpha_m}_{m=0,1,\dots}$.
The following properties of this segmentation will be used repeatedly.
\begin{lemma}
    \label{lem lr bounds}
    For all $m \ge 0$ and $t \ge t_m$, we have $\alpha_t \le T_m^2$.
\end{lemma}
The proof is provided in Section~\ref{sec proof lem lr bounds}.

\begin{lemma}
    \label{lem lr bounds 2}
    For all $m \ge 0$, we have $\bar \alpha_m \le 2 T_m$.
\end{lemma}
The proof is provided in Section~\ref{sec proof lem lr bounds 2}.
We do note that the above two lemmas are analogous to Lemmas 1 \& 2 of \citet{qian2024sureconvergenceratesconcentration} but the analysis is more challenging due to the requirement of $\eta > \frac{1}{2(1-\tau)}$.
    Following \citet{qian2024sureconvergenceratesconcentration}, we now investigate the iterates $\qty{w_t}$ interval by interval.
    Telescoping~\eqref{eq direct gtd} yields
    \begin{align}
        \label{eq w telescope}
        \textstyle w_{t_{m+1}} = w_{t_m} + \sum_{t=t_m}^{t_{m+1}-1} \alpha_t(-\hat{A}_{t+f(t)}^\top \hat{A}_t w_t - \hat{A}_{t+f(t)}^\top \hat b_t),
    \end{align}
    where we have used shorthand $\hat A_t \doteq \rho_t x_t(\gamma x_{t+1} - x_t)^\top$ and $\hat b_t \doteq \rho_t R_{t+1} x_t$.
    For ease of presentation,
    we define for all $m > 0$,
    \begin{align}
        \label{eq q m}
        q_m = w_{t_m} + A^{-1} b.
    \end{align}
    Then, our goal is to show that $\qty{q_m}$ converges to 0.
    Plugging in~\eqref{eq q m} into~\eqref{eq w telescope} yields
\begin{align}
    q_{m+1} &= \textstyle q_m + \sum_{t=t_m}^{k_{m+1}-1} \alpha_t(-\hat{A}_{t+f(t)}^\top \hat{A}_t w_t -\hat{A}_{t+f(t)}^\top \hat{b}_t)\\
    &= \textstyle q_m + \sum_{t=t_m}^{k_{m+1}-1} \alpha_t\left[-\hat{A}_{t+f(t)}^\top \hat{A}_t \left(w_t+A^{-1}b\right) - \hat{A}_{t+f(t)}^\top \left(\hat{b}_t - \hat{A}_t A^{-1}b\right)\right]\\
    &= q_m + g_{1,m} + g_{2,m} + g_{3,m} + g_{4,m},
\end{align}
where
\begin{align}
    g_{1,m} =& \textstyle\sum_{t=t_m}^{t_{m+1}-1} \alpha_t (-A^\top A q_m) = - \bar{\alpha}_m A^\top A q_m, \\
    g_{2,m} =& \textstyle\sum_{t=t_m}^{t_{m+1}-1} \alpha_t \left(A^\top A - \E\left[\hat{A}_{t+f(t)}^\top \hat{A}_t | \fF_{t_m + f(t_m)} \right]\right) q_m \\
    &\textstyle- \sum_{t=t_m}^{t_{m+1}-1} \alpha_t \E\left[\hat{A}_{t+f(t)}^\top \left(\hat{b}_t - \hat{A}_t A^{-1}b\right) | \fF_{t_m + f(t_m)} \right], \\
    g_{3,m} =& \textstyle\sum_{t=t_m}^{t_{m+1}-1} \alpha_t \left(\E\left[\hat{A}_{t+f(t)}^\top \hat{A}_t | \fF_{t_m + f(t_m)} \right] - \hat{A}_{t+f(t)}^\top \hat{A}_t\right) q_m \\
    &\textstyle+ \sum_{t=t_m}^{t_{m+1}-1} \alpha_t \left(\E\left[\hat{A}_{t+f(t)}^\top \left(\hat{b}_t - \hat{A}_t A^{-1}b\right) | \fF_{t_m + f(t_m)} \right] -  \hat{A}_{t+f(t)}^\top \left(\hat{b}_t - \hat{A}_t A^{-1}b\right) \right), \\
    g_{4,m} =&\textstyle \sum_{t=t_m}^{t_{m+1}-1} \alpha_t \hat{A}_{t+f(t)}^\top \hat{A}_t  \left[q_m - \left(w_t + A^{-1} b\right)\right].
\end{align}
Here $\fF_t$ denotes the $\sigma$-algebra until time $t$,
i.e.,
    $\fF_t \doteq \sigma\qty(w_0, S_0, A_0, \dots, S_{t-1}, A_{t-1}, S_t)$.
We use the following lemmas to bound each term above.
In \citet{qian2024sureconvergenceratesconcentration},
they do not have terms like $f(t)$ (cf. $f(t) = 0$).
As a result,
their $w_{t+1}$ is adapted to $\fF_t$.
But in our analysis,
due to the dependence on $\hat{A}_{t+f(t)}$,
$w_{t+1}$ is \emph{not} adapted to $\fF_t$ and is only adapted to $\fF_{t+f(t)}$.
This greatly complicates the analysis, and we will repeatedly use Lemma~\ref{lem stochastic estimates} to address this challenge.
Moreover, Assumption~\ref{assu feature} implies that the matrix $A^\top A$ is positive definite, i.e.,
there exists a constant $\beta > 0$ such that for all $w$,
\begin{align}
    \label{eq pd beta}
    w^\top A^\top A w \geq \beta \norm{w}^2.
\end{align}
This $\beta$ plays a key role in the following bounds.
The finiteness of the MDP and Assumptions~\ref{assu chain} \&~\ref{assu feature} ensure the existence of a constant $H < \infty$ such that
\begin{align}
    \textstyle \sup_t\max\qty{\norm{\hat A_t}, \norm{\hat b_t}, \norm{A}, \norm{b}, \norm{\hat A_t A^{-1}b}} \leq H.
\end{align}

\begin{lemma}
    \label{lem qm wt distance}
    If $e^{2 T_m H^2} \le 2$, then for all $t$ such that $t_m \le t < t_{m+1}$, we have
    \begin{align}
        \norm{ q_m - \left(w_t + A^{-1} b\right) } \le 8 T_m H^2 (\| q_m \| +1).
    \end{align}
\end{lemma}
The proof is provided in Section~\ref{sec proof lem qm wt distance}.

\begin{lemma}\label{lem gm1 bound}
    If $2 H^4 T_m \le \beta$, then $\| q_m + g_{1,m}\|^2 \le (1 - \beta T_m) \|q_m\|^2$.
\end{lemma}
The proof is provided in Section~\ref{sec proof lem gm1 bound}.

\begin{lemma}\label{lem gm2 bound}
    If $T_m \le 1$, then
    \begin{align}
        \textstyle \norm{g_{2,m}} \le C_M T_m^{2(1-\tau)} (H+1) \left(C_\tau + L(f, \chi) + \frac{1}{1-\chi}\right) (\norm{q_m} + 1),
    \end{align} 
    where $L(f, \chi) = \sum_{t=0}^\infty \chi^{f(t)}$, $C_\tau$ is defined in Assumption~\ref{assu gap},
    and $C_M$ is defined in Lemma~\ref{lem stochastic estimates}.
    Notably, $L(f, \cdot)$ is finite due to Assumption \ref{assu gap}.
\end{lemma}
The proof is provided in Section~\ref{sec proof lem gm2 bound}.

\begin{lemma}\label{lem gm3 bound}
    $\| g_{3,m} \| \le 8 T_m H^2 (\| q_m \| +1)$ and $\E\left[g_{3,m} | \fF_{t_m + f(t_m)}\right] = 0$.    
\end{lemma}
The proof is provided in Section~\ref{sec proof lem gm3 bound}.

\begin{lemma} \label{lem gm4 bound}
    If $e^{2 T_m H^2} \le 2$, then $\| g_{4,m} \| \le 8 T_m^2 H^4 (\| q_m \| +1)$.
\end{lemma}
The proof is provided in Section~\ref{sec proof lem gm4 bound}.
Putting all the bounds together,
the following lemma shows that 
the sequence $\qty{\norm{q_m}}_{m\ge0}$ is a supermartingale sequence.
\begin{lemma}\label{lem qm supermartingale property}
    If $T_m \le \min\left(\frac{\beta}{2 H^4}, 1, \frac{\ln(2)}{2 H^2}\right)$, 
    then there exists a scalar $D$ such that
    \begin{align}
        \textstyle \E\left[\norm{q_{m+1}}^2 | \fF_{t_m + f(t_m)}\right] \le \left(1-\beta T_m + D T_m^{2(1-\tau)}\right) \| q_m \|^2 + D T_m^{2(1-\tau)},
    \end{align}
    where $\tau$ is defined in Assumption~\ref{assu gap}.
    In particular, when $D T_m^{2(1-\tau)} \le \frac{1}{2} \beta T_m$, we have
    \begin{align}
        \label{eq qm sup}
        \textstyle \E\left[\norm{q_{m+1}}^2 | \fF_{t_m + f(t_m)}\right] \le \left(1- \frac{1}{2} \beta T_m\right) \| q_m \|^2 + D T_m^{2(1-\tau)}.
    \end{align}
\end{lemma}
The proof is provided in Section~\ref{sec proof lem qm supermartingale property}.
The supermartingale convergence theorem can then take over to show the convergence of $\qty{q_m}$.

\begin{lemma}\label{lem qm as convergence}
    $\lim_{m\to\infty}\norm{q_m} = 0$ a.s.  
\end{lemma}
The proof is provided in Section~\ref{sec proof lem qm as convergence}.
With all the established lemmas, 
we can draw our final conclusion using Lemma \ref{lem qm wt distance}. 
Since both $T_m$ and $q_m$ converges to 0 almost surely, 
the difference between $w_t + A^{-1} b$ and $q_m$ converges to 0 almost surely. 
As a result, we can conclude that $\qty{w_t + A^{-1} b}_{t=0,1,\dots}$ converges to 0, i.e., 
$\qty{w_t}$ converges to $-A^{-1}b$ almost surely,
which completes the proof.
\end{proof}

\section{Finite Sample Analysis of $\attd$}
Theorem~\ref{thm asym} proves the asymptotic convergence of~\eqref{eq direct gtd}.
The price we pay is a memory of size $\Omega\left(\ln^2 t\right)$ (cf.~\eqref{eq memory}).
If the convergence is fast, e.g., $\fO\left(\frac{1}{t}\right)$,
the memory increases reasonably slowly, and we argue that the memory overhead is acceptable.
If,
however,
the convergence is too slow,
the memory may still become too large.
To make sure that~\eqref{eq direct gtd} is a practical algorithm,
therefore,
requires performing a finite sample analysis.
To this end,
we, in this section, provide a finite sample analysis of a variant of~\eqref{eq direct gtd},
which adopts an additional projection operator
and updates $\qty{w_t}$ iteratively as
\begin{align}
    \label{eq p direct gtd}
    \delta_t \doteq& R_{t+1} + \gamma x_{t+1}^\top w_t - x_t^\top w_t \\
    w_{t+1} \doteq& \Gamma\left(w_t + \alpha_t \rho_{t+f(t)} \left(x_{t+f(t)} - \gamma x_{t+f(t) + 1}\right) x^\top_{t+f(t)}  \rho_t \delta_t x_t\right),
\end{align}
where $\Gamma: \R^K \to \R^K$ is a projection operator onto a ball of a radius $B$.
The update~\eqref{eq p direct gtd} differs from~\eqref{eq direct gtd} only in that it adopts an additional projection operator~$\Gamma$.
We, therefore, call it Projected $\attd$.
We show that the convergence rate of our Projected $\attd$ is on par with the convergence rate of the canonical on-policy linear TD in~\citet{bhandari2018finite},
up to a few logarithm terms.
Notably, adding a projection operator is a common practice in finite sample analysis of TD algorithms (see, e.g., \citet{liu2015finite,wang2017finite,bhandari2018finite,zou2019finite}) to simplify the presentation.
Techniques from~\citet{srikant2019finite} can indeed be used to perform finite sample analysis of the original~\eqref{eq direct gtd}.
This, however, complicates the presentation, and we, therefore, leave it for future work.
We now present our main results.

\begin{theorem}
    \label{thm finite}
    Let $B$ be large enough such that $\norm{w_*} \leq B$.
    Consider learning rates in the form of
        $\alpha_t = \frac{C_\alpha}{t+1}$.
    Let Assumptions~\ref{assu chain},~\ref{assu feature}, \&~\ref{assu gap} hold.
    Then there exists a constant $C_0$ such that as long as $C_\alpha \geq C_0$,
    the iterates $\qty{w_t}$ generated by Projected Direct GTD~\eqref{eq p direct gtd} satisfy
    \begin{align}
        \textstyle \E\left[\norm{w_t - w_*}^2\right] = \fO\left(\frac{f(t)\ln(t)}{t}\right).
    \end{align}
\end{theorem}
The proof of Theorem~\ref{thm finite} is provided in Section~\ref{sec proof of thm finite}.
Theorem~\ref{thm finite},
together with~\eqref{eq gap function example},
confirms that the convergence rate of Projected $\attd$ is reasonably fast.
In particular,
if the configuration in~\eqref{eq lr example} is used,
the convergence rate of Projected $\attd$ is on-par with the on-policy linear TD \citep{bhandari2018finite} up to logarithmic factors.

\section{Experiments}
\begin{figure}[h]
    \centering
    \includegraphics[width=\textwidth]{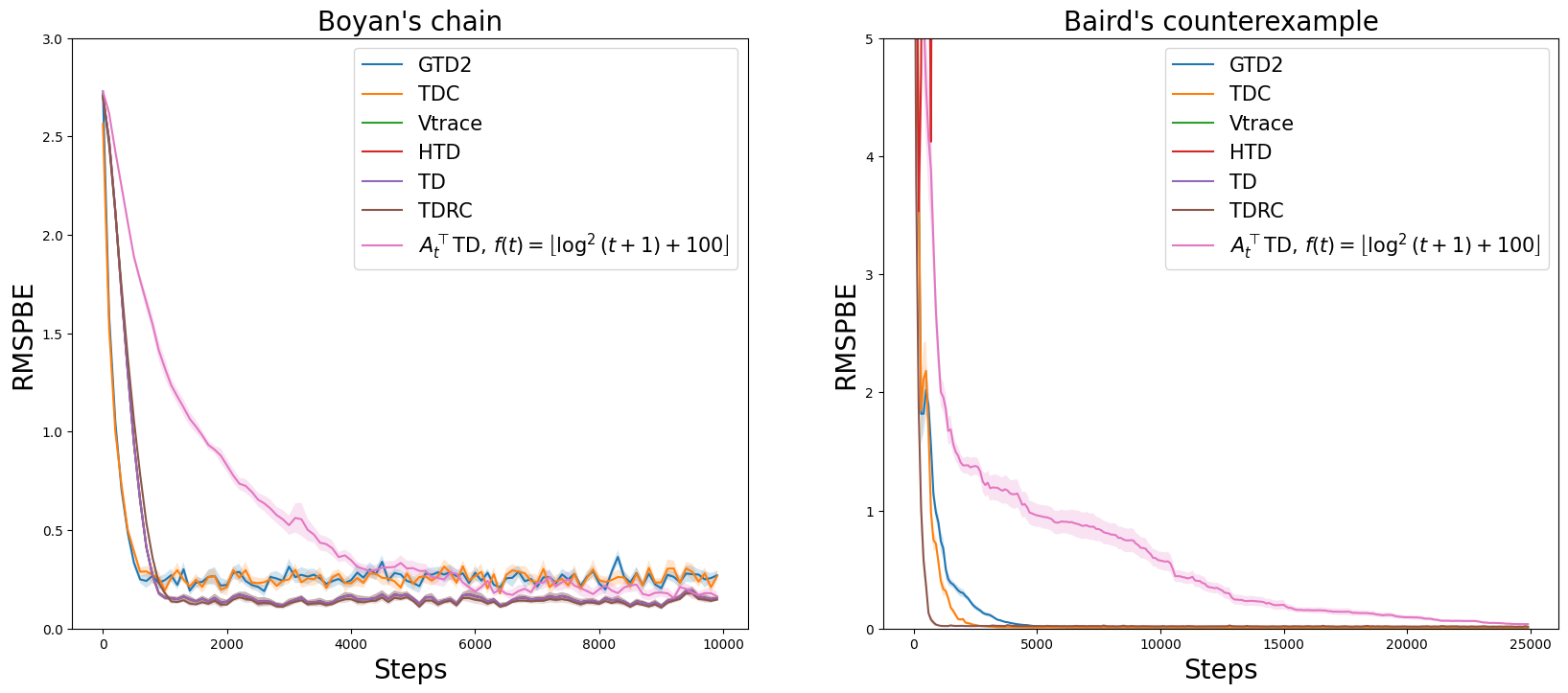}
    \caption{\label{fig baselines} Comparison of~\eqref{eq direct gtd} with previous TD algorithms. All curves are averaged over 10 random seeds with shaded regions showing standard errors. The curves of Vtrace and HTD are invisible in Boyan's chain because they reduce to TD in the on-policy setting. The curves of Vtrace, HTD, and TD are almost invisible in Baird's counterexample because they diverge very quickly.}
\end{figure}
We now empirically compare~\eqref{eq direct gtd} with a few other TD algorithms with linear function approximation,
including (naive) off-policy TD, GTD2 \citep{sutton2009fast}, TDC \citep{sutton2009fast},
Vtrace \citep{espeholt2018impala},
HTD \citep{white2016investigating},
and TDRC \citep{ghiassian2020gradient}.
Those baselines are also used in \citet{ghiassian2020gradient}.
We consider two benchmark tasks,
Boyan's chain \citep{boyan2002technical}
and Baird's counterexample \citep{baird1995residual},
which are also used in \citet{ghiassian2020gradient}.
Notably, Boyan's chain is an on-policy problem while Baird's counterexample is an off-policy problem.
Following \citet{ghiassian2020gradient},
we report the square root of the mean squared projected Bellman error (RMSPBE) at each time step.
We base our implementation on the open-sourced implementation from \citet{ghiassian2020gradient}.
So we refer the reader to \citet{ghiassian2020gradient} for details of the baselines and the tasks,
as well as the exact definition of RMSPBE.
For each algorithm,
we tune its learning rate in $\qty{2^{-20}, \dots, 2^{-1}, 1}$ and report the results with the best learning rate (in terms of minimizing RMSPBE at the last step).
As can be seen in Figure~\ref{fig baselines},
in the on-policy setting,~\eqref{eq direct gtd} outperforms both GTD2 and TDC in terms of the final performance.
In the off-policy setting,~\eqref{eq direct gtd} remains convergent and achieves the same final performance as GTD2, TDC, and TDRC.

\begin{figure}[h]
    \centering
    \includegraphics[width=\textwidth]{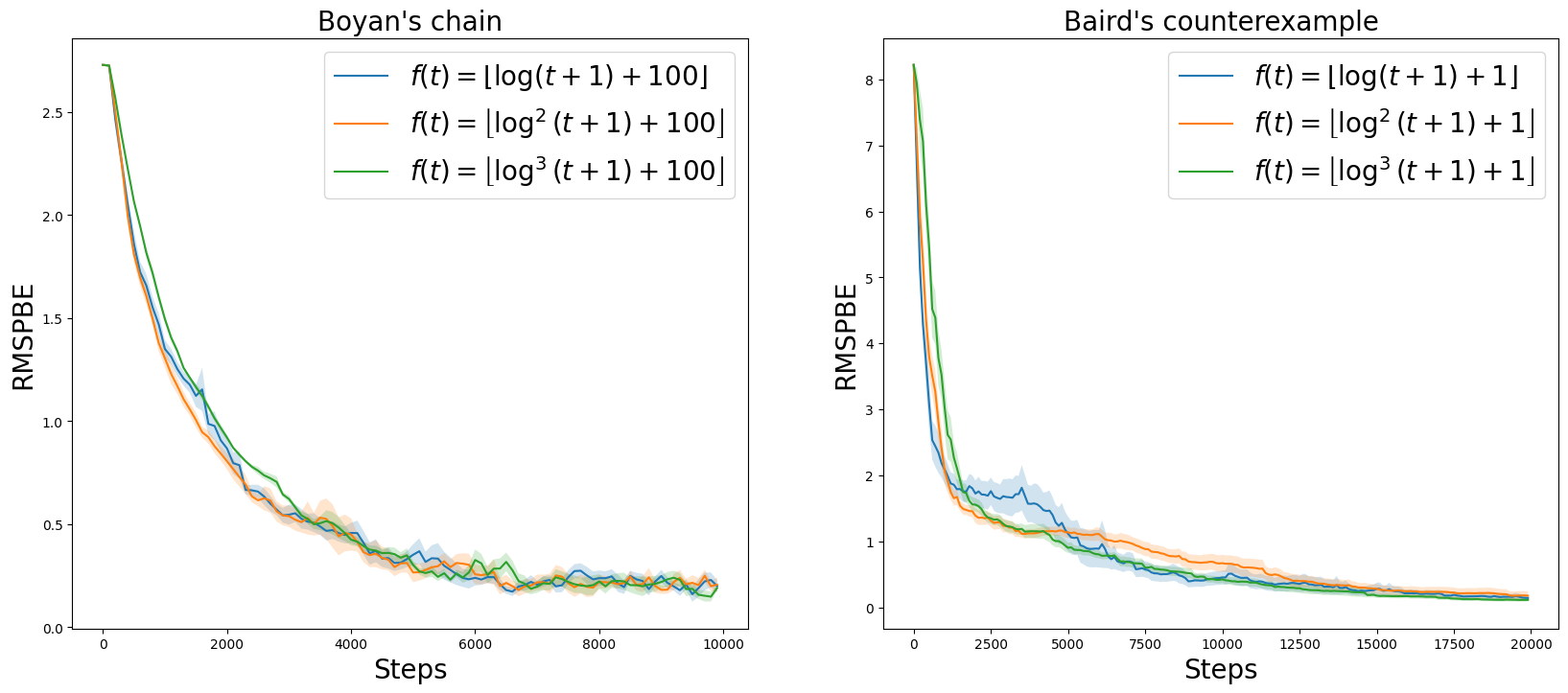}
    \caption{\label{fig gaps} \eqref{eq direct gtd} with different gap functions. All curves are averaged over 10 random seeds with shaded regions showing standard errors. 
    }
\end{figure}
We also investigate different choices of the gap function.
As can be seen in Figure~\ref{fig gaps},
even when we set the gap function to $f(t) = \fO\qty(\log t)$,
which does not respect Assumption~\ref{assu gap},
~\eqref{eq direct gtd} remains convergent and achieves a similar performance.
This suggests that Assumption~\ref{assu gap} might be overly conservative.
Pushing Assumption~\ref{assu gap} to its limit and finding the slowest gap function is an interesting and rewarding future work. 
We conjecture that if certain knowledge of the transition function of the MDP can be incorporated into the design of the gap function,
the gap function can be greatly slowed down.

\section{Related Work}
It is worth mentioning that the memory in $\attd$ is conceptually different from the buffer for experience replay \citep{lin1992self,mnih2015human},
though both store previous transitions.
A replay buffer is typically used to sample mini-batch data randomly.
But $\attd$ only deterministically uses the first and the last entries in the memory.

Instead of minimizing NEU, \citet{feng2019kernel} propose to minimize a kernel loss based on the Bellman error (cf. \citet{baird1995residual}).
In the Markovian setting we consider in this paper,
\citet{feng2019kernel} develop a gradient estimator of this kernel loss.
This gradient estimator is consistent as the size of the replay buffer grows to infinity.
However,
other than the consistency,
no convergence analysis is provided.
Indeed,
as demonstrated in Section~\ref{sec as convergence},
analyzing the almost sure convergence of such algorithms is extremely challenging, and no standard stochastic approximation tools apply.
We conjecture that our new analysis techniques in Section~\ref{sec as convergence} could further help build an almost sure convergence of their algorithms under certain kernels with some moderate regularization.

GTD is one of the many possible methods to address the deadly triad issue.
Other methods include emphatic TD methods \citep{mahmood2015emphatic,sutton2016emphatic,hallak2016generalized,zhang2021truncated,he2023loosely},
target networks \citep{zhang2021breaking,fellows2023target,che2024target},
and density ratio methods \citep{hallak2017consistent,liu2018breaking,nachum2019dualdice,zhang2020gradientdice}.
We refer the reader to \citet{ghiassian2021empirical,ghiassian2024off} for more thorough empirical study of off-policy prediction algorithms.

\section{Conclusion}
In this paper,
we revisit the derivation of the seminal GTD algorithm.
We demonstrate that the idea behind the $\atd$ algorithm can lead to a new off-policy policy evaluation algorithm $\attd$ that is as competitive as GTD in terms of both asymptotic convergence, convergence rate, and per-step computation cost.
$\attd$ does incur additional memory cost, which we argue is negligible in any empirical implementations.
The main advantage of $\attd$ over GTD is that it has only one set of parameters and one learning rate.
It is well documented that the two learning rates in GTD are hard to tune in many empirical problems.
As said in \citet{sutton2009convergent}, ``we are still exploring different ways of setting the step-size parameters'' (of GTD).
It is worth noting again that the main contribution of this work is the rediscovery of the $\atd$ idea,
leading to the $\attd$ algorithm.
That being said,
the empirical study in this work is only preliminary and we leave a more thorough empirical study for future work.

\section*{Acknowledgements}
This work is supported in part by the US National Science Foundation (NSF) under grants III-2128019 and SLES-2331904.

\bibliography{bibliography}
\bibliographystyle{iclr2025_conference}
\newpage
\appendix

\section{Proof of Theorem~\ref{thm finite}}

\label{sec proof of thm finite}
\begin{proof}
    We first define a few shorthands. We use $\bar g(w)$ and $g_t(w)$ to denote the true gradient and its stochastic estimate at time $t$, respectively, i.e.,
    \begin{align}
        \bar{g}(w) =& - A^\top (A w + b), \\
        g_t(w) =& - \hat{A}_{t+f(t)}^\top (\hat{A}_t w + \hat{b}_t).
    \end{align}
    We further define
    \begin{align}
        \Lambda_t(w) = \langle w - w_*, g_t(w) - \bar{g}(w) \rangle,
    \end{align}
    where we recall that $w^*$ is defined in~\eqref{eq td fixed point}.
    The following lemma states several useful properties of the functions defined above.
\begin{lemma}\label{lem finite sample bounds}
    There exist constants $C_g$ and $C_{Lip}$ such that for all $w$, $w'$ with $\norm{w}, \norm{w'} \le B$, we have
    \begin{align}
        &\max\qty{g_t(w), \bar g(w), \Lambda_t(w)} \le C_g, \\
        &\langle w-w', \bar g(w) - \bar g(w') \rangle \le - \beta \norm{w - w'}^2, \\
        &\max\qty{\norm{g_t(w) - g_t(w')}, \norm{\bar g(w) - \bar g(w')}, |\Lambda_t(w) - \Lambda_t(w')|} \le C_{Lip} \norm{w - w'}.
    \end{align}
    We recall that $\beta$ is defined in~\eqref{eq pd beta}.
\end{lemma}
The proof is provided in section~\ref{sec proof lem finite sample bounds}.

\begin{lemma}\label{lem expected lambda bound}
    For all time $t$ and $0<k<t$, we have the following bound
    \begin{align}
        \norm{\E\left[\Lambda_t(w_t)\right]} \le C_g C_\alpha C_{Lip} \ln\left(\frac{t}{t-k}\right) + 2B(B+1) C_M \begin{cases}
            1 & k < f(t-k)\\
            \chi^{f(t)} + \chi^{t - (t-k + f(t-k))} & k \ge f(t-k)
            \end{cases}.
    \end{align}
\end{lemma}
The proof is provided in section~\ref{sec proof lem expected lambda bound}.

We are now ready to decompose the error as
\begin{align}
    &\E\left[\norm{w_{t+1} - w_*}^2\right] \\
    \le& \E\left[\norm{\Gamma(w_t + \alpha_t g_t(w_t)) - w_*}^2\right] \\
    \le& \E\left[\norm{\Gamma(w_t + \alpha_t g_t(w_t)) - \Gamma(w_*)}^2\right] \\
    \le& \E\left[\norm{w_t + \alpha_t g_t(w_t) - w_*}^2\right] \\
    =& \E\left[\norm{w_t  - w_*}^2 + \alpha_t^2 \norm{g_t(w_t)}^2 + 2\alpha_t \langle w_t - w_*, g_t(w_t) \rangle\right].
\end{align}
Since $\bar g(w_*) = 0$,
we have
\begin{align}
&\langle w_t - w_*, g_t(w_t) \rangle \\
=& \langle w_t - w_*, g_t(w_t) - \bar g(w_t) + \bar g(w_t) - \bar g(w_*) \rangle \\
=& \Lambda_t(w_t) + \langle w_t - w_*, \bar g(w_t) - \bar g(w_*)\rangle,
\end{align}
yielding
\begin{align}
    &\E\left[\norm{w_{t+1} - w_*}^2\right] \\
    \le& \E\left[\norm{w_t  - w_*}^2 + \alpha_t^2 \norm{g_t(w_t)}^2 + 2\alpha_t \langle w_t - w_*, g_t(w_t) \rangle\right] \\
    \le& \E\left[\norm{w_t  - w_*}^2 + \alpha_t^2 \norm{g_t(w_t)}^2 + 2\alpha_t \Lambda_t(w_t) + 2\alpha_t \langle w_t - w_*, \bar g(w_t) - \bar g(w_*)\rangle\right].
\end{align}
Applying Lemma~\ref{lem finite sample bounds}, we get
\begin{align}
    &\E\left[\norm{w_{t+1} - w_*}^2\right] \\
    \le& \E\left[\norm{w_t  - w_*}^2 + \alpha_t^2 C_g^2 + 2\alpha_t \Lambda_t(w_t) - 2\alpha_t \beta \norm{w_t - w_*}^2\right] \\
    =& (1 - 2 \alpha_t \beta) \E[\norm{w_t  - w_*}^2] + \alpha_t^2 C_g^2 + 2\alpha_t \E[\Lambda_t(w_t)].
\end{align}
Plugging in $\alpha_t = \frac{C_\alpha}{t+1}$ and multiplying both sides by $(t+1)$, we have
\begin{align}
    &(t+1) \E\left[\norm{w_{t+1} - w_*}^2\right] \\
    \le& (t+1) \left(1 - 2 \beta \frac{C_\alpha}{t+1}\right) \E[\norm{w_t  - w_*}^2] + (t+1) \left(\frac{C_\alpha}{t+1}\right)^2 C_g^2 + (t+1) \cdot 2 \frac{C_\alpha}{t+1} \E[\Lambda_t(w_t)] \\
    \le& \left(t+ 1 - 2 \beta C_\alpha \right) \E[\norm{w_t  - w_*}^2] + \frac{C_\alpha^2 C_g^2}{t+1} + 2 C_\alpha \E[\Lambda_t(w_t)].
\end{align}

Let $C_0 = \frac{1}{2 \beta}$. Then as $C_\alpha \ge C_0$, $2\beta C_\alpha \ge 2 \beta C_0 = 1$. Hence, $t+1-2\beta C_\alpha \le t$. Since $\E\left[\norm{w_{t+1} - w_*}^2\right] \ge 0$, we have
\begin{align}
    (t+1) \E\left[\norm{w_{t+1} - w_*}^2\right] \le t \E\left[\norm{w_{t+1} - w_*}^2\right] + \frac{C_\alpha^2 C_g^2}{t+1} + 2 C_\alpha \E[\Lambda_t(w_t)]
\end{align}

Applying the inequality recursively,
\begin{align}
    &T \E\left[\norm{w_T - w_*}^2\right] \\
    \le& \sum_{t=0}^{T-1} \left(\frac{C_\alpha^2 C_g^2}{t+1} + 2 C_\alpha \E[\Lambda_t(w_t)]\right) \\
    =& C_\alpha^2 C_g^2 \sum_{t=1}^{T} \frac{1}{t} + 2 C_\alpha \sum_{t=0}^{k} \E[\Lambda_t(w_t)] + 2 C_\alpha \sum_{t=k+1}^{T-1} \E[\Lambda_t(w_t)] \\
    \le& C_\alpha^2 C_g^2 [\ln(T)+1] + 2 C_\alpha \sum_{t=0}^{k} \E[\Lambda_t(w_t)] + 2 C_\alpha \sum_{t=k+1}^{T-1} \E[\Lambda_t(w_t)],
\end{align}
where the last inequality comes from the bound for harmonic numbers.
Applying Lemma~\ref{lem finite sample bounds} again yields
\begin{align}
    &T \E\left[\norm{w_{T} - w_*}^2\right] \\
    \le& C_\alpha^2 C_g^2 [\ln(T)+1] + 2 C_\alpha \sum_{t=0}^{k} C_g + 2 C_\alpha \sum_{t=k+1}^{T-1} \E[\Lambda_t(w_t)] \\
    \label{eq T times expected distance}
    =& C_\alpha^2 C_g^2 [\ln(T)+1] + 2 k C_\alpha C_g + 2 C_\alpha \sum_{t=k+1}^{T-1} \E[\Lambda_t(w_t)].
\end{align}
Now, we will use Lemma~\ref{lem expected lambda bound} to bound the last summation. 
Firstly, we take
\begin{align}
   k \doteq 2f(T). 
\end{align}
Secondly, Assumption \ref{assu gap} suggests that for some $\tau \in (0, \frac{3\nu-2}{\nu})$,
\begin{align}
    f(t) \le C_\tau \alpha_t^{-\tau} = C_\tau \left(\frac{C_\alpha}{(t+1)^\nu}\right)^{-\tau} = \frac{C_\tau (t+1)^{\tau \nu}}{C_\alpha^\tau}. 
\end{align}
Since $\nu \le 1$, $\tau \nu < \frac{3\nu-2}{\nu} = 3\nu-2 \le 1$. Hence, there exists a constant $T_0$ such that for all $T \ge T_0$, we have
\begin{align}
    2 f(T) \le \frac{2 C_\tau (T+1)^{\tau \nu}}{C_\alpha^\tau} < T.
\end{align}
For the rest of the argument, we will assume that $T \ge T_0$, and we will then have $2f(T) < T$. 

As $f$ is increasing, 
for all $t \in (k, T)$, 
we have
\begin{align}
   k = 2 f(T) \ge 2f(t) \geq f(t-k). 
\end{align}
So for $t > k$,
\begin{align}
    \E[\Lambda_t(w_t)] &\le C_g C_\alpha C_{Lip} \ln\left(\frac{t}{t-k}\right) + 2B(B+1) C_M \left( \chi^{f(t)} + \chi^{k - f(t-k)} \right)\explain{Lemma~\ref{lem expected lambda bound}}\\
    &\le C_g C_\alpha C_{Lip} \ln\left(\frac{t}{t-k}\right) + 4B(B+1) C_M \chi^{f(t)},
\end{align}
where the last inequality uses the fact that for all $t \in (k ,T)$,
\begin{align}
  k - f(t-k) = 2 f(T) - f(t - k) > 2 f(t) - f(t) = f(t).  
\end{align}
Summing them up then yields
\begin{align}
    &\sum_{t=k+1}^{T-1} \E[\Lambda_t(t)] \\
    \le& \sum_{t=k+1}^{T-1} \left(C_g C_\alpha C_{Lip} \ln\left(\frac{t}{t-k}\right) + 4B(B+1) C_M \chi^{f(t)}\right) \\
    \le& C_g C_\alpha C_{Lip} \sum_{t=k+1}^{T-1}(\ln(t) - \ln(t-k)) + 4B(B+1) C_M \sum_{t=k+1}^T \chi^{f(t)} \\
    \le& C_g C_\alpha C_{Lip} \left(\sum_{t=k+1}^{T-1} \ln(t) - \sum_{t=k+1}^{T-1} \ln(t-k)\right) + 4B(B+1) C_M \sum_{t=0}^\infty \chi^{f(t)} \\
    \le& C_g C_\alpha C_{Lip} \left(\sum_{t=k+1}^{T-1} \ln(t) - \sum_{t=1}^{T-k-1} \ln(t)\right) + 4B(B+1) L(f, \chi) \\
    \le& C_g C_\alpha C_{Lip} \left(\sum_{t=T-k}^{T-1} \ln(t) - \sum_{t=1}^k \ln(t)\right) + 4B(B+1) L(f, \chi) \\
    \le& C_g C_\alpha C_{Lip} \left(\sum_{t=T-k}^{T-1} \ln(t) - \ln[t - (T-k-1)]\right) + 4B(B+1) L(f, \chi) \\
    \le& C_g C_\alpha C_{Lip} \sum_{t=T-k}^{T-1} \log\left[\frac{t}{t-(T-k-1)}\right] + 4B(B+1) L(f, \chi) \\
    \le& C_g C_\alpha C_{Lip} \sum_{t=T-k}^{T-1} \log(T) + 4B(B+1) L(f, \chi) \\
    \label{eq sum of lambda_t}
    \le& C_g C_\alpha C_{Lip} k \log(T) + 4B(B+1) L(f, \chi),
\end{align}
where second last inequality holds because 
\begin{align}
  \log\left[\frac{t}{t-(T-k-1)}\right] \le \log(T)  
\end{align}
holds for all $T-k-1 < t < T$
and $L(f, \chi) = \sum_{t=0}^\infty \chi^{f(t)} < \infty$ due to Assumption~\ref{assu gap}.

Plugging \eqref{eq sum of lambda_t} into \eqref{eq T times expected distance} then yields
\begin{align}
    &T \E[\norm{w_T - w_*}^2] \\
    \le& C_\alpha^2 C_g^2 [\ln(T)+1] + 2 k C_\alpha C_g + 2 C_\alpha [C_g C_\alpha C_{Lip} k \log(T) + 4B(B+1) L(f, \chi)] \\
    =& C_\alpha^2 C_g^2 [\ln(T)+1] + 2 k C_\alpha C_g + 2 C_g C_\alpha^2 C_{Lip} k \log(T) + 8 C_\alpha B(B+1) L(f, \chi),
\end{align}
Since we have defined $k \doteq 2f(T)$,
we have that for $T \ge T_0$,
\begin{align}
    &\E[\norm{w_T - w_*}^2] \\
    \le& \frac{C_\alpha^2 C_g^2 [\ln(T)+1]}{T} + \frac{4 f(T) C_\alpha C_g}{T} + \frac{4 C_g C_\alpha^2 C_{Lip} f(T) \log(T)}{T} + \frac{8 C_\alpha B(B+1) L(f, \chi)}{T}.
\end{align}
Thus, for all $T$, we have
\begin{align}
    \E[\norm{w_T - w_*}^2] = \fO\left(\frac{f(T)\ln(T)}{T}\right),
\end{align}
which completes the proof.
\end{proof}

\section{Auxiliary Lemmas}

\begin{lemma}
    \label{lem stochastic estimates}
    Let Assumption~\ref{assu chain} hold. 
    Then there exists a constant $C_M > 0$ and $\chi \in [0, 1)$ such that
    \begin{align}
        \label{eq matrix converge1}
        &\norm{\E\left[\hat A_t \right] - A} \leq C_M \chi^t, \\
        \label{eq matrix converge2}
        &\norm{\E\left[\hat A_{t+k}^\top \hat A_t\right] - \E\left[\hat A_{t+k}\right]^\top \E\left[\hat A_t\right]} \leq C_M \chi^k, \\
        \label{eq matrix converge3}
        &\norm{\E\left[\hat A^\top_{t+k} \hat A_t | \fF_l\right] - A^\top A} \leq C_M  \begin{cases}
            1 & t < l\\
            \chi^{k} + \chi^{t - l} & t \ge l
        \end{cases}.
    \end{align}
    Similarly,
    \begin{align}
        &\norm{\E\left[\hat b_t \right] - A} \leq C_M \chi^t, \\
        &\norm{\E\left[\hat A_{t+k}^\top \hat b_t\right] - \E\left[\hat A_{t+k}\right]^\top \E\left[\hat b_t\right]} \leq C_M \chi^k, \\
        &\norm{\E\left[\hat A^\top_{t+k} \hat b_t | \fF_l\right] - A^\top b} \leq C_M  \begin{cases}
            1 & t < l\\
            \chi^{k} + \chi^{t - l} & t \ge l
        \end{cases}. \\
    \end{align}
\end{lemma}
\begin{proof}
    For the simplicity of display, 
    we include the proof only for the first half of the lemma.
    The second half is identical up to change of notations and is, therefore, omitted to avoid verbatim.

    Define an augmented chain $\qty{Y_t}$ evolving in
    \begin{align}
        \fY \doteq \qty{(s, a, s') \in \fS \times \fA \times \fS \mid d_\mu(s) > 0, \mu(a|s) > 0, p(s'|s, a) > 0}
    \end{align}
    as
    \begin{align}
        Y_t \doteq (S_t, A_t, S_{t+1}).
    \end{align}
    According to the definition of $\fF_t$,
    it can be easily seen that $Y_t$ is adapted to $\fF_{t+1}$.
    Assumption~\ref{assu chain} immediately ensures that $\qty{Y_t}$ is also ergodic with a stationary distribution
    \begin{align}
        d_Y(y) = d_\mu(s)\mu(a|s)p(s'|s, a).
    \end{align} 
    Here we have used $y$ as shorthand for $(s, a, s')$.
    Define functions
    \begin{align}
        \hat A(y) \doteq& \rho(s, a)x(s)\left(\gamma x(s') - x(s)\right)^\top, \\
        \hat b(y) \doteq& \rho(s, a)x(s)r(s, a).
    \end{align}
    It can then be easily computed that 
    \begin{align}
        \hat A_t =& \hat A(Y_t), \\
        A =& \sum_y d_Y(y)\hat A(y).
    \end{align}
    Assumption~\ref{assu chain} ensures that the chain $\qty{Y_t}$ mixes geometrically fast.
    In other words,
    there exist constants $\chi \in [0, 1)$ and $C_0 > 0$ such that for any $t$ and $k$,
    \begin{align}
        \max_y \sum_{y'} \abs{\Pr(Y_{t+k}  = y' | Y_t = y) - d_Y(y')} \leq C_0 \chi^k.
    \end{align} 
    This is a well-known result, and we refer the reader to Theorem 4.7 of \citet{levin2017markov} for detailed proof.
    Then we have 
    \begin{align}
        &\norm{\E\left[\hat A_t\right] - A} \\
        =&\norm{\sum_y \Pr(Y_t = y)\hat A(y) - \sum_y d_Y(y) \hat A(y)} \\
        \leq& \sum_y \norm{\Pr(Y_t = y) - d_Y(y)}\norm{\hat A(y)} \\
        \leq& \max_y\norm{\hat A(y)}\sum_y \norm{\Pr(Y_t = y) - d_Y(y)} \\
        \leq& HC_0 \chi^t,
    \end{align}
    which completes the proof of~\eqref{eq matrix converge1}.
    Similarly, we have 
    \begin{align}
        & \norm{\E\left[\hat A_{t+k}^\top \hat A_t\right] - \E\left[\hat A_{t+k}\right] \E\left[\hat A_t\right]}\\
        =& \norm{\sum_{y'} \sum_y \Pr(Y_{t+k} = y', Y_t = y) \hat A(y')^\top \hat A(y) - \left(\sum_{y'} \Pr(Y_{t+k} = y') \hat A(y') \right)^\top \left(\sum_y \Pr(Y_t = y) \hat A(y)\right) } \\
        =& \norm{\sum_{y'} \sum_y \left(\Pr(Y_{t+k} = y', Y_t = y) - \Pr(Y_{t+k} = y') \Pr(Y_t = y)\right) A(y')^\top \hat A(y)} \\
        \le& \max_{y'} \norm{\hat A(y')^\top} \max_y \norm{\hat A(y)} \norm{\sum_{y'} \sum_y \left(\Pr(Y_{t+k} = y', Y_t = y) - \Pr(Y_{t+k} = y') \Pr(Y_t = y)\right)} \\
        \le& H \cdot H \sum_{y'} \sum_{y} \abs{\Pr(Y_{t+k} = y', Y_t = y) - \Pr(Y_{t+k} = y') \Pr(Y_t = y)} \\
        \le& H^2 \sum_{y'} \sum_{y} \abs{\Pr(Y_{t+k} = y' | Y_t = y) \Pr(Y_t = y) - \Pr(Y_{t+k} = y') \Pr(Y_t = y)} \\
        \le& H^2 \sum_{y'} \sum_{y} \abs{\Pr(Y_{t+k} = y' | Y_t = y) - \Pr(Y_{t+k} = y')}  \Pr(Y_t = y) \\
        \le& H^2 \sum_{y'} \sum_{y} \abs{\Pr(Y_{t+k} = y' | Y_t = y) - \Pr(Y_{t+k} = y')} \\
        \le& H^2 \sum_{y'} \sum_{y} \abs{\Pr(Y_{t+k} = y' | Y_t = y) - d_Y(y')| + |d_Y(y') - \Pr(Y_{t+k} = y')} \\
        \le& H^2 \sum_{y} \left(\sum_{y'} |\Pr(Y_{t+k} = y' | Y_t = y) - d_Y(y')| + \sum_{y'} |\Pr(Y_{t+k} = y') - d_Y(y')| \right) \\
        \le& H^2 \sum_{y} \left(C_0 \chi^k + C_0 \chi^{t+k} \right) \\
        \le& 2 H^2 |\fY| C_0  \chi^k,
    \end{align}
    which proves \eqref{eq matrix converge2}.
    This also suggests $\forall l$,
    \begin{align}
        \norm{\E\left[\hat A_{t+k}^\top \hat A_t\mid \fF_l\right] - \E\left[\hat A_{t+k} \mid \fF_l\right]^\top \E\left[\hat A_t\mid \fF_l\right]} \leq 2H^2 \abs{\fY}C_0 \chi^k.
    \end{align}
    To see this, we consider the two cases of whether $l<t$ separately. 
    
    \textbf{Case 1}: $l < t$. Then, by the Markov property,
    \begin{align}
        &\norm{\E\left[\hat A_{t+k}^\top \hat A_t\mid \fF_l\right] - \E\left[\hat A_{t+k} \mid \fF_l\right]^\top \E\left[\hat A_t\mid \fF_l\right]} \\
        =& \norm{\E\left[\hat A_{t+k}^\top \hat A_t\mid Y_l\right] - \E\left[\hat A_{t+k} \mid Y_l\right]^\top \E\left[\hat A_t\mid Y_l\right]} \\
        \le& 2H^2 |\fY|C_0 \chi^k.
    \end{align}
    
    \textbf{Case 2}: $l \ge t$. Then $\hat A_t$ is deterministic given $\fF_l$. and
    \begin{align}
        &\norm{\E\left[\hat A_{t+k}^\top \hat A_t\mid \fF_l\right] - \E\left[\hat A_{t+k} \mid \fF_l\right]^\top \E\left[\hat A_t\mid \fF_l\right]} \\
        =&\norm{\E\left[\hat A_{t+k}^\top \E\left[\hat A_t \mid \fF_l \right]\mid \fF_l\right] - \E\left[\hat A_{t+k} \mid \fF_l\right]^\top \E\left[\hat A_t\mid \fF_l\right]} \\
        =&\norm{\E\left[\hat A_{t+k} \mid \fF_l\right]^\top \E\left[\hat A_t \mid \fF_l \right] - \E\left[\hat A_{t+k} \mid \fF_l\right]^\top \E\left[\hat A_t\mid \fF_l\right]} \\
        =& 0.
    \end{align}
    Lastly, combining the geometrical convergence suggested in \eqref{eq matrix converge1} and the geometrically decaying correlation implied by \eqref{eq matrix converge2}, we can prove \eqref{eq matrix converge3} in the following manner.
    First, 
    \begin{align}
        &\norm{\E\left[\hat A^\top_{t+k} \hat A_t | \fF_l\right] - A^\top A} \\
        \le& \norm{\E\left[\hat A^\top_{t+k} \hat A_t | \fF_l\right] - \E\left[\hat A_{t+k} | \fF_l\right]^\top \E\left[\hat A_t | \fF_l\right] + \E\left[\hat A_{t+k} | \fF_l\right]^\top \E\left[\hat A_t | \fF_l\right] - A^\top A} \\
        \le& \norm{\E\left[\hat A^\top_{t+k} \hat A_t | \fF_l\right] - \E\left[\hat A_{t+k} | \fF_l\right]^\top \E\left[\hat A_t | \fF_l\right]} + \norm{\E\left[\hat A_{t+k} | \fF_l\right]^\top \E\left[\hat A_t | \fF_l\right] - A^\top A} \\
        \le& 2 H^2 |\fY| C_0 \chi^k  + \norm{\E\left[\hat A_{t+k} | \fF_l\right]^\top \E\left[\hat A_t | \fF_l\right] - \E\left[\hat A_{t+k} | \fF_l\right]^\top A + \E\left[\hat A_{t+k} | \fF_l\right]^\top A - A^\top A} \\
        \le& 2 H^2 |\fY| C_0 \chi^k + \norm{\E\left[\hat A_{t+k} | \fF_l\right]^\top (\E\left[\hat A_t | \fF_l\right] - A)} + \norm{(\E\left[\hat A_{t+k} | \fF_l\right] - A)^\top A} \\
        \le& 2 H^2 |\fY| C_0 \chi^k + \norm{\E\left[\hat A_{t+k} | \fF_l\right]^\top} \norm{\E\left[\hat A_t | \fF_l\right] - A} + \norm{(\E\left[\hat A_{t+k} | \fF_l\right] - A)^\top} \norm{A} \\
        \le& 2 H^2 |\fY| C_0 \chi^k + H \norm{\E\left[\hat A_t | \fF_l\right] - A} + H \norm{\E\left[\hat A_{t+k} | \fF_l\right] - A} \\
        \le& 2 H^2 |\fY| C_0 \chi^k + H\left(\norm{\E\left[\hat A_t | \fF_l\right] - A} + \norm{\E\left[\hat A_{t+k} | \fF_l\right] - A}\right).
    \end{align}
    We now bound the last term.
    For $t < l$, we use the trivial bound
    \begin{align}
        \norm{\E\left[\hat A_t | \fF_l\right] - A} + \norm{\E\left[\hat A_{t+k} | \fF_l\right] - A} &\le 4H.
    \end{align}
    For $t \ge l$, both $Y_t$ and $Y_{t+k}$ are not adapted to $\fF_l$.
    We, therefore, have 
    \begin{align}
        \norm{\E\left[\hat A_t | \fF_l\right] - A} + \norm{\E\left[\hat A_{t+k} | \fF_l\right] - A}
        &\le H C_0 \chi^{t-l} + H C_0 \chi^{t+k-l} \le 2 H C_0 \chi^{t-l}.
    \end{align}
    Combining the results, we obtain
    \begin{align}
        \norm{\E\left[\hat A^\top_{t+k} \hat A_t | \fF_l\right] - A^\top A} &\leq 2 H^2 |\fY| C_0 \chi^{k} + \begin{cases}
            4 H^2 & t < l\\
            2 H^2 C_0 \chi^{t - l} & t \ge l
        \end{cases} \\
        &=\begin{cases}
            2 H^2 |\fY| C_0 + 4 H^2 & t < l\\
            2 H^2 |\fY| C_0 \chi^{k} + 2 H^2 C_0 \chi^{t - l} & t \ge l
        \end{cases},\explain{Since $\chi \le 1$}
    \end{align}
    which completes the proof of~\eqref{eq matrix converge3}.
\end{proof}

\begin{lemma}\label{lem sufficient condition for gap assu}
    If the gap function $f(t) = \floor{h(t) \ln(t)}$ where $h$ is a non-negative, increasing function tending to infinity,
    then $\sum_{t=0}^\infty \chi^{f(t)}$ for all $\chi \in (0,1)$.
\end{lemma}

\begin{proof}
    Firstly, we should note that for all $t$, $\floor{h(t) \ln(t)} > h(t) \ln(t) - 1$.
    Therefore, take arbitrary $\chi \in (0,1)$,
    \begin{align}
        \sum_{t=0}^\infty \chi^{f(t)} \le \sum_{t=0}^\infty \chi^{h(t)\ln(t)-1} = \frac{1}{\chi} \sum_{t=0}^\infty e^{\ln(\chi) h(t)\ln(t)} =  \frac{1}{\chi} \sum_{t=0}^\infty t^{\ln(\chi) h(t)}.
    \end{align}

    Since $h$ is increasing in $t$ and tending to infinity, there exists a $T$ such that got all $t \ge T$, $f(t) \ge -\frac{2}{\ln(\chi)}$. Then, for all $t \ge T$, $t^{\ln(\chi) h(t)} \le t^{-2}$. Thus, by comparison test and p-test, we can conclude that $\sum_{t=0}^\infty \chi^{f(t)} \le \sum_{t=0}^\infty t^{-2} < \infty$.
\end{proof}

\section{Proof of Technical Lemmas}
\label{sec aux lemmas}


\subsection{Proof of Lemma~\ref{lem lr bounds}}
\label{sec proof lem lr bounds}
\begin{proof}
    We proceed via induction on $m$. 
    In particular, 
    we prove the following two inequalities for all $m$:
    \begin{align}
        \label{eq induction1}
        t_m \ge& \frac{m^{\eta + 1}}{16 \max(C_\alpha,1)}, \\
        \label{eq induction2}
        \alpha_t \le& T_m^2, \forall t \ge t_m.
    \end{align}

    \textbf{Base Case m=0}: Obviously we have
    \begin{align}
      t_0 = 0 = \frac{0^{\eta + 1}}{16 \max(C_\alpha,1)},  
    \end{align}
    so \eqref{eq induction1} holds for $m=0$.
    Since $\eta \in (0,1]$, 
    we have
    \begin{align}
      T_0 = \frac{16 \max(C_\alpha, 1)}{\eta + 1} \ge 8 \max(C_\alpha, 1).
    \end{align}
    Hence, for all $t \ge t_0 = 0$, 
    \begin{align}
        \alpha_t = \frac{C_\alpha}{(t+1)^\nu} \le C_\alpha \le 8 \max(C_\alpha, 1)^2 \leq T_0^2. 
    \end{align}
    So \eqref{eq induction2} holds.

    \textbf{Induction Step}: Suppose \eqref{eq induction1} and \eqref{eq induction2} hold for $m=k$.
    We now verify them for $m=k+1$.
    Letting $m=k$ in~\eqref{eq induction2} yields
    \begin{align}
        T_k \le& \sum_{t=t_k}^{t_{k+1}-1} \alpha_t \explain{Defintion of $\qty{t_m}$ in~\eqref{eq def tm}} \\
        \le& \sum_{t=t_k}^{t_{k+1}-1} T_k^2 = (t_{k+1} - t_k) T_k^2. 
    \end{align}
    Dividing both sides by $T_k$ yields
    \begin{align}
        t_{k+1} - t_k \ge \frac{1}{T_k} = \frac{(\eta+1) (k+1)^\eta}{16 \max(C_\alpha,1)}.
    \end{align}
    Consequently, we have
    \begin{align}
        t_{k+1} \ge& t_k + \frac{(\eta+1) (k+1)^\eta}{16 \max(C_\alpha,1)}  \\
        \ge& \frac{k^{\eta+1}}{16 \max(C_\alpha,1)} + \frac{(\eta+1) (k+1)^\eta}{16 \max(C_\alpha,1)},
    \end{align}
    where the last inequality results from inductive hypothesis~\eqref{eq induction1}.
    Since $\frac{(\eta+1) (k+1)^\eta}{16 \max\qty{C_\alpha, 1}}$ is monotonically increasing in $k$,
    we have 
    \begin{align}
        \frac{(\eta+1) (k+1)^\eta}{16 \max(C_\alpha,1)} \ge \int_k^{k+1} \frac{(\eta+1) (k+1)^\eta}{16 \max(C_\alpha,1)} = \frac{(k+1)^{\eta+1}}{16 \max(C_\alpha,1)} - \frac{k^{\eta+1}}{16 \max(C_\alpha,1)}.
    \end{align}
    We have thus verified \eqref{eq induction1} for $m=k+1$, i.e.
    \begin{align}
        t_{k+1} \ge \frac{k^{\eta+1}}{16 \max(C_\alpha,1)} + \frac{(k+1)^{\eta+1}}{16 \max(C_\alpha,1)} - \frac{k^{\eta+1}}{16 \max(C_\alpha,1)} = \frac{(k+1)^{\eta+1}}{16 \max(C_\alpha,1)}.
    \end{align}
    To verify \eqref{eq induction2} for $m=k+1$, 
    we will make use of our just proven~\eqref{eq induction1} with $m=k+1$.
    Take arbitrary $t \ge t_{k+1}$.
    As $\alpha_t = \frac{C_\alpha}{(t+1)^\nu}$ is a monotonically decreasing in $t$, we have
    \begin{align}
        \alpha_t \le \alpha_{t_{k+1}} = \frac{C_\alpha}{(t_{k+1}+1)^\nu} \le \frac{C_\alpha}{t_{k+1}^\nu}.
    \end{align}
    Using~\eqref{eq induction1} with $m=k+1$, we get
    \begin{align}
        \alpha_t \le \frac{C_\alpha}{t_{k+1}^\nu} \le \frac{C_\alpha}{\left(\frac{1}{16 \max(C_\alpha,1)} (k+1)^{\eta+1}\right)^\nu} \le \frac{16^\nu C_\alpha \max(C_\alpha,1)^\nu}{(k+2)^{(\eta+1)\nu}} \left(\frac{k+2}{k+1}\right)^{(\eta+1)\nu}.
    \end{align}
    As $\frac{k+2}{k+1} \le 2$ for $k \ge 0$ and $\eta, \nu \in (0, 1]$, we have
    \begin{align}
      \left(\frac{k+2}{k+1}\right)^{(\eta+1)\nu} \le 2^{(\eta+1)\nu}.  
    \end{align}
    Moreover, since $C_\alpha \le \max(C_\alpha, 1)$,
    we have 
    \begin{align}
      C_\alpha \max(C_\alpha, 1)^\nu \le \max(C_\alpha, 1)^{1+\nu}.  
    \end{align}
    Thus we have
    \begin{align}
        \alpha_t \le \frac{16^\nu C_\alpha \max(C_\alpha,1)^\nu}{(k+2)^{(\eta+1)\nu}} \left(\frac{k+2}{k+1}\right)^{(\eta+1)\nu} \le \frac{2^{(\eta+5)\nu} \max(C_\alpha,1)^{1+\nu}}{(k+2)^{(\eta+1)\nu}}.
    \end{align}
    Using $\eta, \nu \in (0, 1]$,
    we have
    \begin{align}
        0 \le& (\eta+5) \nu \le 6, \\
        \max(C_\alpha,1)^{1+\nu} \le& \max(C_\alpha,1)^2.
    \end{align}
    The definition of $\eta$ in~\eqref{eq def eta} implies
    \begin{align}
        2\eta < (\eta+1)\nu.
    \end{align}
    Hence, we get
    \begin{align}
        \alpha_t \le \frac{2^{(\eta+5)\nu} \max(C_\alpha,1)^{1+\nu}}{(k+2)^{(\eta+1)\nu}} \le \frac{64 \max(C_\alpha,1)^2}{(k+2)^{2\eta}}.
    \end{align}
    The second inequality in \eqref{eq def eta} together with the fact that $\nu \in (0, 1]$ implies that $\eta <1$.
    Consequently, we have $(\eta+1)^2 < 4$. 
    Therefore, $64 < \frac{256}{(1+\eta)^2}$ and
    \begin{align}
        \alpha_t \le \frac{64 \max(C_\alpha,1)^2}{(k+2)^{2 \eta}} \le \frac{256 \max(C_\alpha,1)^2}{(\eta+1)^2 (k+2)^{2 \eta}} = \left( \frac{16 \max(C_\alpha,1)}{(\eta+1) (k+2)^\eta} \right)^2 = T_{k+1}^2. 
    \end{align}
    We have now verified that~\eqref{eq induction2} holds for $m=k+1$,
    which completes the induction.
\end{proof}

\subsection{Proof of Lemma \ref{lem lr bounds 2}}
\label{sec proof lem lr bounds 2}
\begin{proof}
    The fact that $\nu \in (0, 1)$ and~\eqref{eq def eta} implies
    \begin{align}
        \eta < \frac{\nu}{2-\nu} \le \nu.
    \end{align}
    The fact that $\eta \in [0, 1]$ implies
    \begin{align}
        \frac{16}{\eta + 1} > 8.
    \end{align}
    Consequently, we have
    \begin{align}
        T_m = \frac{16}{\eta+1} \frac{\max(C_\alpha, 1)}{(m+1)^\eta} > 8 \frac{C_\alpha}{(m+1)^\nu} = 8 \alpha_m.
    \end{align}
    The definition of $\qty{t_m}$ in~\eqref{eq def tm} implies that $t_{m+1}-t_m \ge 1$ for all $m \ge 0$,
    so we have
    \begin{align}
        t_m \geq m.
    \end{align}
    Moreover, because $\alpha_t = \frac{C_\alpha}{(t+1)^\nu}$ is decreasing in $t$, for all $t \ge t_m \ge m$,
    we have
    \begin{align}
        \alpha_t \le \alpha_m \le \frac{T_m}{8}.
    \end{align}
    The definition of $\qty{t_m}$ in~\eqref{eq def tm} also implies that $\sum_{t=t_m}^{t_{m+1}-2} \alpha_t < T_m$.
    Then we have
    \begin{align}
        \bar \alpha_m = \sum_{t=t_m}^{t_{m+1}-1} \alpha_t =  \sum_{t=t_m}^{t_{m+1}-2} \alpha_t + \alpha_{t_{m+1}-1} \le T_m + \frac{T_m}{8} = \frac{9 T_m}{8} \le 2 T_m,
    \end{align}
    which completes the proof.
    Note here we have used the convention that $\sum_{t=i}^j \alpha_t \doteq 0$ if $i > j$.
\end{proof}

\subsection{Proof of Lemma~\ref{lem qm wt distance}}
\label{sec proof lem qm wt distance}
\begin{proof}
    For all $t \ge 0$,
    \begin{align}
        &\norm{w_{t+1} + A^{-1} b} \\
        =& \norm{\left(w_t + A^{-1} b\right)  + \alpha_t \hat{A}_{t+f(t)}^\top \left(\hat{A}_t (w_t + A^{-1} b) - \hat A_t A^{-1} b + \hat{b}_t\right) } \\
        \le& \norm{ w_t + A^{-1} b } + \alpha_t \norm{ \hat{A}_{t+f(t)}^\top} \left(\norm{\hat{A}_t} \norm{w_t + A^{-1} b} + \norm{\hat{A}_t A^{-1} b} + \|\hat{b}_t\|\right)\\
        \le& \norm{ w_t + A^{-1} b } + \alpha_t H \left(H \norm{ w_t + A^{-1} b } + H + H\right) \\
        =&\label{eq wtm+1 in terms of wtm} \norm{ w_t + A^{-1} b } + \alpha_t H^2 \left(\norm{ w_t + A^{-1} b } + 2 \right)
    \end{align}
    Therefore, by adding $2$ to both sides, we get
    \begin{align}
        \norm{w_{t+1} + A^{-1} b} + 2 \le (1 + \alpha_t H^2) \left(\norm{w_t + A^{-1} b} + 2\right).
    \end{align}
    Applying the inequality iteratively, we have that for all $t$ satisfying $t_m \le t \le t_{m+1}$
    \begin{align}
        \norm{w_t + A^{-1} b}+2 \le \left(\norm{w_{t_m} + A^{-1} b} + 2\right) \prod_{j=t_m}^{t_{m+1}} (1 + \alpha_t H^2) \le e^{\bar{\alpha}_m H^2} \left(\norm{w_{t_m} + A^{-1} b} + 2\right),
    \end{align}
    where for the last two inequalities, we used the fact that
    \begin{align}
        \prod_{j=t_m}^{t_{m+1}-1} (1 + \alpha_t H^2) \le \exp\left(\sum_{j=t_m}^{t_{m+1}-1} \alpha_t H^2 \right) = \exp\left(\bar{\alpha}_m H^2 \right).
    \end{align}
    As $\bar{\alpha}_m \le 2 T_m$ (Lemma~\ref{lem lr bounds 2}) and $e^{2 T_m H^2} \le 2$,
    \begin{align}
        \norm{w_t + A^{-1} b}+2 &\le e^{\bar{\alpha}_m H^2} \left(\norm{w_{t_m} + A^{-1} b} + 2\right) \\
        &\le e^{2 T_m H^2} \left(\norm{w_{t_m} + A^{-1} b} + 2\right) \\
        &\le 2 \left(\norm{w_{t_m} + A^{-1} b} + 2\right).
    \end{align}
    Hence, 
    \begin{align}
        \label{eq wt bound in terms of wtm}
        \norm{w_t + A^{-1} b} \le 2 \left(\norm{w_{t_m} + A^{-1} b} + 1\right).
    \end{align}
    Therefore, for all $t_m \le t \le t_{m+1}$, we have
    \begin{align}
        &\norm{\left(w_t + A^{-1} b\right) - q_m} \\
        =& \norm{\left(w_t + A^{-1} b\right) - \left(w_{t_m} + A^{-1} b\right)} \\
        \le& \norm{ \sum_{j=t_m}^{t-1} \left(w_{j+1} + A^{-1} b\right) - \left(w_j + A^{-1} b\right) } \\
        \le& \sum_{j=t_m}^{t-1} \norm{ \left(w_{j+1} + A^{-1} b\right) - \left(w_j + A^{-1} b\right) } \\        
        =& \sum_{j=t_m}^{t-1} \norm{ \alpha_j \hat{A}_{j+f(j)}^\top \left(\hat{A}_j (w_t + A^{-1} b) - \hat A_j A^{-1} b + \hat{b}_j\right) } \\        
        \le& \sum_{j=t_m}^{t_{m+1}-1} \alpha_t H^2 \left(\norm{w_t + A^{-1} b}+2\right)
        \explain{Similar to \eqref{eq wtm+1 in terms of wtm}} \\
        \le& \sum_{j=t_m}^{t_{m+1}-1} \alpha_t H^2 \left(2 \norm{w_{t_m} + A^{-1} b } + 4\right)
        \explain{Using \eqref{eq wt bound in terms of wtm}} \\
        =& 2 \bar{\alpha}_m H^2 \left(\norm{w_{t_m} + A^{-1} b } + 2\right) \\
        \le& 4 T_m H^2 \left(\norm{w_{t_m} + A^{-1} b } + 2\right) \explain{Using Lemma~\ref{lem lr bounds 2}} \\
        \le& 8 T_m H^2 \left(\norm{w_{t_m} + A^{-1} b } + 1\right) \\
        =& 8 T_m H^2 \left(\norm{q_m} + 1\right),
    \end{align}
    which completes the proof.
\end{proof}

\subsection{Proof of Lemma \ref{lem gm1 bound}}
\label{sec proof lem gm1 bound}
\begin{proof}
    \begin{align}
        \| q_m + g_{1,m}\|^2 &= \| q_m - \bar{\alpha}_m A^\top A q_m \|^2 = \norm{ \left(I - \bar{\alpha}_m A^\top A\right) q_m }^2 \\
        &= q_m^\top \left(I - \bar{\alpha}_m A^\top A \right)^\top \left(I - \bar{\alpha}_m A^\top A \right) q_m \\
        &= \| q_m \|^2 - 2 \bar{\alpha}_m q_m^\top A^\top A q_m + \bar{\alpha}_m^2 \norm{ A^\top A q_m }^2 \\
        &\le \| q_m \|^2 - 2 \bar{\alpha}_m \beta \|q_m\|^2 + \bar{\alpha}_m^2 \norm{ A^\top }^2 \| A \|^2 \|q_m\|^2 \\
        &\le \| q_m \|^2 - 2 \bar{\alpha}_m \beta \|q_m\|^2 + \bar{\alpha}_m^2 H^4 \|q_m\|^2 \\
        &\le \| q_m \|^2 - 2 \bar{\alpha}_m \beta \|q_m\|^2 + 2 \bar{\alpha}_m T_m H^4 \|q_m\|^2 \explain{Lemma~\ref{lem lr bounds 2}} \\
        &\le \| q_m \|^2 - 2 \bar{\alpha}_m \beta \|q_m\|^2 + \bar{\alpha}_m \beta \|q_m\|^2 \explain{Assumption of this Lemma} \\
        &\le \left(1 - \beta \bar \alpha_m \right)\| q_m \|^2 \\
        &\le \left(1 - \beta T_m \right)\| q_m \|^2 \explain{Definition of $T_m$ in~\eqref{eq def tm}}. 
    \end{align}
\end{proof}

\subsection{Proof of Lemma \ref{lem gm2 bound}}
\label{sec proof lem gm2 bound}
\begin{proof}
    \begin{align}
        \norm{g_{2,m}} \le& \norm{\sum_{t=t_m}^{t_{m+1}-1} \alpha_t \left(A^\top A - \E\left[\hat{A}_{t+f(t)}^\top \hat{A}_t | \fF_{t_m + f(t_m)} \right]\right) q_m} \\
        &+ \norm{\sum_{t=t_m}^{t_{m+1}-1} \alpha_t \E\left[\hat{A}_{t+f(t)}^\top \left(\hat{b}_t - \hat{A}_t A^{-1}b\right) | \fF_{t_m + f(t_m)} \right]} \\
        \le& \sum_{t=t_m}^{t_{m+1}-1} \alpha_t \norm{A^\top A - \E\left[\hat{A}_{t+f(t)}^\top \hat{A}_t | \fF_{t_m + f(t_m)} \right]} \norm{q_m} \\
        &+ \sum_{t=t_m}^{t_{m+1}-1} \alpha_t \norm{\E\left[\hat{A}_{t+f(t)}^\top \left(\hat{b}_t - \hat{A}_t A^{-1}b\right) | \fF_{t_m + f(t_m)} \right]} \\
        \le& \sum_{t=t_m}^{t_{m+1}-1} \alpha_t \norm{A^\top A - \E\left[\hat{A}_{t+f(t)}^\top \hat{A}_t | \fF_{t_m + f(t_m)} \right]} \norm{q_m} \\
        &+ \sum_{t=t_m}^{t_{m+1}-1} \alpha_t \norm{\E\left[\hat{A}_{t+f(t)}^\top \hat{b}_t - A^\top b | \fF_{t_m + f(t_m)}\right]} \\
        &+ \sum_{t=t_m}^{t_{m+1}-1} \alpha_t \norm{\E\left[ A^\top b - \hat{A}_{t+f(t)}^\top \hat{A}_t A^{-1}b | \fF_{t_m + f(t_m)}\right]} \\
        =& \sum_{t=t_m}^{t_{m+1}-1} \alpha_t \norm{A^\top A - \E\left[\hat{A}_{t+f(t)}^\top \hat{A}_t | \fF_{t_m + f(t_m)} \right]} \norm{q_m} \\
        &+ \sum_{t=t_m}^{t_{m+1}-1} \alpha_t \norm{\E\left[\hat{A}_{t+f(t)}^\top \hat{b}_t - A^\top b | \fF_{t_m + f(t_m)}\right]} \\
        &+ \sum_{t=t_m}^{t_{m+1}-1} \alpha_t \norm{\E\left[ A^\top A - \hat{A}_{t+f(t)}^\top \hat{A}_t | \fF_{t_m + f(t_m)}\right] A^{-1} b} \\
        \le& \sum_{t=t_m}^{t_{m+1}-1} \alpha_t \norm{A^\top A - \E\left[\hat{A}_{t+f(t)}^\top \hat{A}_t | \fF_{t_m + f(t_m)} \right]} \norm{q_m} \\
        &+ \sum_{t=t_m}^{t_{m+1}-1} \alpha_t \norm{\E\left[\hat{A}_{t+f(t)}^\top \hat{b}_t - A^\top b | \fF_{t_m + f(t_m)}\right]} \\
        &+ \sum_{t=t_m}^{t_{m+1}-1} \alpha_t \norm{\E\left[ A^\top A - \hat{A}_{t+f(t)}^\top \hat{A}_t | \fF_{t_m + f(t_m)}\right]} \norm{A^{-1} b}.
    \end{align}
    To bound the last three terms, we will consider separately whether $t_m + f(t_m) < t_{m+1}$. 
    If $t_m + f(t_m) \ge t_{m+1}$, then applying Lemma \ref{lem stochastic estimates} yields
    \begin{align}
        \norm{g_{2,m}} 
        \le& \sum_{t=t_m}^{t_{m+1}-1} \alpha_t C_M \norm{q_m} + \sum_{t=t_m}^{t_{m+1}-1} \alpha_t C_M + \sum_{t=t_m}^{t_{m+1}-1} \alpha_t C_M H \\
        =& C_M (\norm{q_m} + H + 1) \sum_{t=t_m}^{t_{m+1}-1} \alpha_t.
    \end{align}
    Since $\alpha_t = \frac{C_\alpha}{(1+t)^\nu}$ is a decreasing function in $t$, we have
    \begin{align}
        \norm{g_{2,m}} & \le C_M (\norm{q_m} + H + 1) \sum_{t=t_m}^{t_{m+1}-1} \alpha_{t_m} \\
        &= C_M (\norm{q_m} + H + 1) (t_{m+1}-t_m) \alpha_{t_m} \\
        &\le C_M (\norm{q_m} + H + 1) f(t_m) \alpha_{t_m}.
    \end{align}
    Assumption \ref{assu gap} suggests $f(t) \le C_\tau \alpha_t^{-\tau}$, so
    \begin{align}
        \norm{g_{2,m}} & \le C_M (\norm{q_m} + H + 1) C_\tau \alpha_t^{-\tau} \alpha_{t_m} \\
        &\le C_M C_\tau (\norm{q_m} + H + 1) \alpha_t^{1-\tau}.
    \end{align}
    Because $\alpha_t \le T_m^2$ for $t \ge t_m$ (Lemma~\ref{lem lr bounds}),
    \begin{align}
        \norm{g_{2,m}} &\le C_M C_\tau (\norm{q_m} + H + 1) T_m^{2(1-\tau)} \\
        &\le C_M C_\tau (H + 1) T_m^{2(1-\tau)} (\norm{q_m} + 1).
    \end{align}
    For the general case where $t_m + f(t_m) < t_{m+1}$, we break the summation $\sum_{t=t_m}^{t_{m+1} - 1}$ into two parts,
    i.e. $\sum_{t=t_m}^{t_m+f(t_m)-1}$ and $\sum_{t=t_m + f(t_m)}^{t_{m+1}-1}$ and apply Lemma \ref{lem stochastic estimates} separately.
    We have
    \begin{align}
        \norm{g_{2,m}} \leq& \sum_{t=t_m}^{t_m + f(t_m)-1} \alpha_t \norm{A^\top A - \E\left[\hat{A}_{t+f(t)}^\top \hat{A}_t \mid \fF_{t_m + f(t_m)} \right]} \norm{q_m} \\
        &+ \sum_{t=t_m + f(t_m)}^{t_{m+1}-1} \alpha_t \norm{A^\top A - \E\left[\hat{A}_{t+f(t)}^\top \hat{A}_t \mid \fF_{t_m + f(t_m)} \right]} \norm{q_m} \\
        &+ \sum_{t=t_m}^{t_m+f(t_m)-1} \alpha_t \norm{\E\left[\hat{A}_{t+f(t)}^\top \hat{b}_t - A^\top b \mid \fF_{t_m + f(t_m)}\right]} \\
        &+ \sum_{t=t_m + f(t_m)}^{t_{m+1}-1} \alpha_t \norm{\E\left[\hat{A}_{t+f(t)}^\top \hat{b}_t - A^\top b \mid \fF_{t_m + f(t_m)}\right]} \\
        &+ \sum_{t=t_m}^{t_m+f(t_m)-1} \alpha_t \norm{\E\left[ A^\top A - \hat{A}_{t+f(t)}^\top \hat{A}_t \mid \fF_{t_m + f(t_m)}\right]}\norm{A^{-1} b} \\
        &+ \sum_{t=t_m + f(t_m)}^{t_{m+1}-1} \alpha_t \norm{\E\left[ A^\top A - \hat{A}_{t+f(t)}^\top \hat{A}_t \mid \fF_{t_m + f(t_m)}\right]}\norm{A^{-1} b} \\
        \le& \sum_{t=t_m}^{t_m + f(t_m)-1} \alpha_t C_M \norm{q_m} + \sum_{t=t_m + f(t_m)}^{t_{m+1}-1} \alpha_t C_M \left(\chi^{f(t)} + \chi^{t - [t_m + f(t_m)]} \right) \norm{q_m} \\
        &+ \sum_{t=t_m}^{t_m+f(t_m)-1} \alpha_t C_M + \sum_{t=t_m + f(t_m)}^{t_{m+1}-1} \alpha_t C_M \left(\chi^{f(t)} + \chi^{t - [t_m + f(t_m)]} \right) \\
        &+ \sum_{t=t_m}^{t_m+f(t_m)-1} \alpha_t C_M H + \sum_{t=t_m + f(t_m)}^{t_{m+1}-1} \alpha_t C_M \left(\chi^{f(t)} + \chi^{t - [t_m + f(t_m)]} \right) H \\
        \leq& C_M (\norm{q_m} + H + 1) \left[\sum_{t=t_m}^{t_m + f(t_m)-1} \alpha_t + \sum_{t=t_m + f(t_m)}^{t_{m+1}-1} \alpha_t \left(\chi^{f(t)} + \chi^{t - [t_m + f(t_m)]} \right)\right].
    \end{align}
    Because $\alpha_t = \frac{C_\alpha}{(1+t)^\nu}$ is a decreasing function in $t$ and $\alpha_t \le T_m^2$ for $t \ge t_m$ (Lemma~\ref{lem lr bounds}), we get
    \begin{align}
        \norm{g_{2,m}} &\leq C_M (\norm{q_m} + H + 1) \left[\sum_{t=t_m}^{t_m + f(t_m)-1} \alpha_{t_m} + T_m^2 \sum_{t=t_m + f(t_m)}^{t_{m+1}-1} \left(\chi^{f(t)} + \chi^{t - [t_m + f(t_m)]} \right)\right] \\
        &\leq C_M (\norm{q_m} + H + 1) \left[f(t_m) \alpha_{t_m} +  T_m^2 \left(\sum_{t=t_m + f(t_m)}^{t_{m+1}-1} \chi^{f(t)} + \sum_{t=t_m + f(t_m)}^{t_{m+1}-1} \chi^{t - [t_m + f(t_m)]} \right)\right] \\
        &\leq C_M (\norm{q_m} + H + 1) \left[f(t_m) \alpha_{t_m} +  T_m^2 \left(\sum_{t=0}^{\infty} \chi^{f(t)} + \sum_{t=0}^{\infty} \chi^t \right)\right]
    \end{align}
    Assumption \ref{assu gap} suggests that $f(t) \le C_\tau \alpha_t^{-\tau}$ and $L(f, \chi) = \sum_{t=0}^\infty \chi^{f(t)} < \infty$. Therefore,
    \begin{align}
        \norm{g_{2,m}} &\leq C_M (\norm{q_m} + H + 1) \left[C_\tau \alpha_{t_m}^{-\tau} \alpha_{t_m} +  T_m^2 \left(L(f, \chi) + \sum_{t=0}^{\infty} \chi^t \right)\right] \\
        &\leq C_M (\norm{q_m} + H + 1) \left[C_\tau \alpha_{t_m}^{1-\tau} +  T_m^2 \left(L(f, \chi) + \frac{1}{1-\chi} \right)\right] \\
        &\leq C_M (\norm{q_m} + H + 1) \left[C_\tau T_m^{2(1-\tau)} + T_m^2 \left(L(f, \chi) + \frac{1}{1-\chi} \right)\right].
    \end{align}
    When $T_m \le 1$,
    we then have
    \begin{align}
        \norm{g_{2,m}} \le C_M T_m^{2(1-\tau)} (H+1) \left(C_\tau + L(f, \chi) + \frac{1}{1-\chi}\right) (\norm{q_m} + 1),
    \end{align}
    which completes the proof.
\end{proof}

\subsection{Proof of Lemma \ref{lem gm3 bound}}
\label{sec proof lem gm3 bound}
\begin{proof}
Firstly, we have
    \begin{align}
        \|g_{3,m}\| \leq& \norm{\sum_{t=t_m}^{t_{m+1}-1} \alpha_t \left(\E\left[\hat{A}_{t+f(t)}^\top \hat{A}_t | \fF_{t_m + f(t_m)} \right] - \hat{A}_{t+f(t)}^\top \hat{A}_t\right) q_m } \\
        &+ \norm{\sum_{t=t_m}^{t_{m+1}-1} \alpha_t \left(\E\left[\hat{A}_{t+f(t)}^\top \left(\hat{b}_t - \hat{A}_t A^{-1}b\right) | \fF_{t_m + f(t_m)} \right] -  \hat{A}_{t+f(t)}^\top \left(\hat{b}_t - \hat{A}_t A^{-1}b\right) \right)} \\
        \leq& \sum_{t=t_m}^{t_{m+1}-1} \alpha_t \left(\norm{\E\left[\hat{A}_{t+f(t)}^\top \hat{A}_t | \fF_{t_m + f(t_m)} \right]} + \norm{\hat{A}_{t+f(t)}^\top} \|\hat{A}_t\|\right) \|q_m\| \\
        &+ \sum_{t=t_m}^{t_{m+1}-1} \alpha_t \left(\norm{\E\left[\hat{A}_{t+f(t)}^\top \left(\hat{b}_t - \hat{A}_t A^{-1}b\right) | \fF_{t_m + f(t_m)} \right]} +  \norm{\hat{A}_{t+f(t)}^\top} \left(\|\hat{b}_t\| + \norm{\hat{A}_t A^{-1}b}\right)\right) \\
        \leq& \sum_{t=t_m}^{t_{m+1}-1} \alpha_t \left(\E\left[\norm{\hat{A}_{t+f(t)}^\top \hat{A}_t } | \fF_{t_m + f(t_m)} \right] + H \cdot H\right) \|q_m\| \\
        &+ \sum_{t=t_m}^{t_{m+1}-1} \alpha_t \left(\E\left[\norm{\hat{A}_{t+f(t)}^\top \left(\hat{b}_t - \hat{A}_t A^{-1}b\right)} | \fF_{t_m + f(t_m)} \right] + H (H + H)\right) \\
        \le& \sum_{t=t_m}^{t_{m+1}-1} \alpha_t \left(\E\left[\norm{\hat{A}_{t+f(t)}^\top} \norm{\hat{A}_t} | \fF_{t_m + f(t_m)} \right] + H^2 \right) \|q_m\| \\
        &+ \sum_{t=t_m}^{t_{m+1}-1} \alpha_t \left(\E\left[\norm{\hat{A}_{t+f(t)}^\top} \left(\norm{\hat{b}_t} + \norm{\hat{A}_t A^{-1}b}\right) | \fF_{t_m + f(t_m)} \right] +  2 H^2\right) \\
        \le& \sum_{t=t_m}^{t_{m+1}-1} \alpha_t \left(H^2 + H^2 \right) \|q_m\| + \sum_{t=t_m}^{t_{m+1}-1} \alpha_t \left(H(H+H) +  2 H^2\right) \\
        \le& 2 \bar{\alpha}_m H^2 (\norm{q_m} + 2) \\
        \le& 4 T_m H^2 (\norm{q_m} + 2) \explain{Lemma~\ref{lem lr bounds 2}} \\
        \le& 8 T_m H^2 (\norm{q_m} + 1).
    \end{align}    
Secondly, 
\begin{align}
 \E\left[g_{3,m} | \fF_{t_m + f(t_m)}\right] = 0   
\end{align}
holds trivially,
which completes the proof.
\end{proof}

\subsection{Proof of Lemma \ref{lem gm4 bound}}
\label{sec proof lem gm4 bound}
\begin{proof}
    \begin{align}
        \|g_{4,m}\| &= \norm{\sum_{t=t_m}^{t_{m+1}-1} \alpha_t \hat{A}_{t+f(t)}^\top \hat{A}_t  \left[q_m - \left(w_t + A^{-1} b\right)\right] } \\
        &\le \sum_{t=t_m}^{t_{m+1}-1} \alpha_t \norm{\hat{A}_{t+f(t)}^\top} \|\hat{A}_t\|  \norm{q_m - \left(w_t + A^{-1} b\right)}.
    \end{align}
    Then, by Lemma \ref{lem qm wt distance}, we have
    \begin{align}
        \|g_{4,m}\| &\le \sum_{t=t_m}^{t_{m+1}-1} \alpha_t H \cdot H \cdot 8 T_m H^2 (\|q_m\| + 1) \\
        &= 8 \bar{\alpha}_m T_m H^4 (\|q_m\| + 1),
    \end{align}
    which completes the proof.
\end{proof}

\subsection{Proof of Lemma \ref{lem qm supermartingale property}}
\label{sec proof lem qm supermartingale property}
\begin{proof}
    \begin{align}
        \norm{q_{m+1}}^2 =& \norm{q_m + g_{1,m} + g_{2,m} + g_{3,m} + g_{4,m}}^2 \\
        =& \norm{q_m + g_{1,m}}^2 + \norm{g_{2,m} + g_{3,m} + g_{4,m}}^2 + 2 (q_m + g_{1,m})^\top (g_{2,m} + g_{3,m} + g_{4,m}) \\
        \le& \norm{q_m + g_{1,m}}^2 + \norm{g_{2,m} + g_{3,m} + g_{4,m}}^2 \\
        &+ 2 (q_m + g_{1,m})^\top g_{3,m} + 2 \norm{q_m + g_{1,m}} \norm{g_{2,m} + g_{4,m}}.
    \end{align}
    Lemma \ref{lem gm1 bound} implies that $\norm{q_m + g_{1,m}}^2 \le (1 - \beta T_m) \norm{q_m}^2$ and $\norm{q_m + g_{1,m}} = \fO(\norm{q_m})$,
    Lemma \ref{lem gm2 bound} suggests that $\norm{g_{2,m}} = \fO\left(T_m^{2(1-\tau)}(\|q_m\|+1)\right)$,
    Lemma \ref{lem gm3 bound} suggests that $\norm{g_{3,m}} = \fO(T_m(\|q_m\|+1))$, and Lemma \ref{lem gm4 bound} suggests that $\norm{g_{4,m}} = \fO\left(T_m^2(\|q_m\|+1)\right)$. Hence,
    \begin{align}
        \norm{g_{2,m} + g_{3,m} + g_{4,m}} \le \norm{g_{2,m}} + \norm{g_{3,m}} + \norm{g_{4,m}} = \fO (T_m (\|q_m\|+1)),
    \end{align}
    and 
    \begin{align}
        \norm{g_{2,m} + g_{4,m}} = \fO \left(T_m^{2(1-\tau)} (\|q_m\|+1)\right).
    \end{align}
    Moreover,
    both $q_m$ and $g_{1,m} = - \bar \alpha_m A^\top A q_m$ are adapted to $\fF_{t_m + f(t_m)}$ and \ref{lem gm3 bound} implies that $\E\left[g_{3,m} | \fF_{t_m + f(t_m)}\right] = 0$.
    We, therefore, have 
    \begin{align}
        \E\left[(q_m + g_{1,m})^\top g_{3,m}| \fF_{t_m + f(t_m)}\right] = 0.
    \end{align}
    Lastly, putting everything together, we have
    \begin{align}
        &\E\left[\norm{q_{m+1}}^2 | \fF_{t_m + f(t_m)}\right] \\
        \le& (1 - \beta T_m) \norm{q_m}^2 + \fO (T_m (\|q_m\|+1))^2 + \fO(\norm{q_m}) \fO \left(T_m^{2(1-\tau)} (\|q_m\|+1)\right) \\
        \leq& (1 - \beta T_m) \norm{q_m}^2 + \fO \left(T_m^{2(1-\tau)} (\|q_m\|+1)^2\right) \\
        \leq& (1 - \beta T_m) \norm{q_m}^2 + \fO \left(T_m^{2(1-\tau)} (\|q_m\|^2+1)\right),
    \end{align}
    where the last inequality comes from the fact that $(\norm{q_m}+1)^2 \le 2(\|q_m\|^2 +1)$.
    In conclusion, there exists a constant $D$ such that
    \begin{align}
        \E\left[\norm{q_{m+1}}^2 | \fF_{t_m + f(t_m)}\right] \le \left(1 - \beta T_m + D T_m^{2(1-\tau)}\right) \norm{q_m}^2 + D T_m^{2(1-\tau)},
    \end{align}
    which completes the proof.
\end{proof}

\subsection{Proof of Lemma \ref{lem qm as convergence}}
\label{sec proof lem qm as convergence}
\begin{proof}
    To prove the lemma, 
    we will invoke a supermartingale convergence theorem stated as follows.
    \begin{theorem}
        \label{thm supermartingale}
        (Proposition 4.2 in \citet{bertsekas1996neuro})
        Let $Y_m$, $X_m$, and $Z_m$, $m \ge 0$ be three sequences of random variables and let $\bar \fF_m$, $m \ge 0$, be sets of random variables such that $\bar \fF_m \subseteq \bar \fF_{m+1}$ for all $m$. Suppose that
        \begin{enumerate}
            \item The random variables $Y_m$, $X_m$, and $Z_m$ are non-negative and are functions of the random variables in $\bar \fF_m$,
            \item For each $m$, we have $\E\left[Y_{m+1} | \bar \fF_m\right] \le Y_m - X_m + Z_m$,
            \item There holds $\sum_{m=0}^\infty Z_m < \infty$.
        \end{enumerate}
        Then, we have $\sum_{m=0}^\infty X_m < \infty$ almost surely, and the sequence $Y_m$ converges almost surely to a non-negative random variable $Y$.
    \end{theorem}
    In our case, we let 
    \begin{align}
      Y_m =& \| q_m \|^2, \\  
      X_m =& \frac{1}{2} \beta T_m \| q_m \|^2, \\
      Z_m =& D T_m^{2(1-\tau)}, \\
      \bar \fF_{m} =& \fF_{t_m + f(t_m)}.
    \end{align}
    The first condition of Theorem~\ref{thm supermartingale} holds trivially.
    For the second condition of Theorem~\ref{thm supermartingale} to hold,
    we rely on~\eqref{eq qm sup} in Lemma~\ref{lem qm supermartingale property}.
    According to the definition of $T_m$ in~\eqref{eq def big tm},
    we have
    \begin{align}
        \lim_{m\to\infty} T_m = 0.
    \end{align}
    As a result,
    the condition 
    \begin{align}
        T_m \le \min\left(\frac{\beta}{2 H^4}, 1, \frac{\ln(2)}{2 H^2}\right)
    \end{align}
    in Lemma~\ref{lem qm supermartingale property} holds for sufficiently large $m$.
    Moreover, since 
    \begin{align}
        \tau <& \frac{3}{2} - \frac{1}{\nu} \explain{Assumption~\ref{assu gap}} \\
        \leq& \frac{3}{2} - \frac{1}{1} \explain{Assumption~\ref{assu lr}} \\
        =&\frac{1}{2},
    \end{align}
    we have
    \begin{align}
        2(1 - \tau) > 1.
    \end{align}
    Consequently,
    the condition
    \begin{align}
        D T_m^{2(1-\tau)} \le \frac{1}{2} \beta T_m,
    \end{align}
    which is equivalent to
    \begin{align}
        T_m^{2(1-\tau) - 1} \leq \frac{\beta}{2D},
    \end{align}
    also holds for sufficiently large $m$ as $\lim_{m\to\infty} T_m=0$.
    Crucially, $D$ is deterministic because $D$ only depends on $H$, $\beta$, $C_\tau$, $C_M$, $L(f, \chi)$, and $\chi$, which are all deterministic quantities.
    Therefore,
    there always exists a finite and deterministic $m_0$ such that the subsequence $\qty{X_m, Y_m, Z_m}_{m \geq m_0}$ verifies the second condition.
    Since $\eta > \frac{1}{2(1-\tau)}$, i.e. $2 (1-\tau) \eta > 1$, by p-test, we can deduce that
    \begin{align}
        \sum_{m=0}^\infty Z_m &= \sum_{m=0}^\infty D T_m^{2(1-\tau)} = D \sum_{m=0}^\infty \left(\frac{16 \max(C_\alpha, 1)}{(\eta + 1) (m+1)^\eta}\right)^{2(1-\tau)} \\
        &= \frac{(16 \max(C_\alpha, 1))^{2(1-\tau)} D}{(\eta+1)^{2(1-\tau)}} \sum_{m=0}^\infty \frac{1}{(\eta+1)^{2(1-\tau)\eta}} < \infty.
    \end{align}
    The third condition of Theorem~\ref{thm supermartingale}, 
    therefore, also holds.
    All the conditions of Theorem~\ref{thm supermartingale} are now verified for the subsequence $\qty{X_m, Y_m, Z_m}_{m\geq m_0}$,
    which implies that the sequence $\qty{\norm{q_m}^2}_{m\geq 0}$ converges
    and
    \begin{align}
        \frac{1}{2} \beta \sum_{m=0}^\infty T_m \norm{q_m}^2 < \infty.
    \end{align}
    As $\eta < 1$, we have
    \begin{align}
      \sum_{m=0}^\infty T_m = \sum_{m=0}^\infty \frac{16 \max(C_\alpha, 1)}{(\eta + 1) (m+1)^\eta} = \frac{16 \max(C_\alpha, 1)}{\eta + 1} \sum_{m=0}^\infty \frac{1}{(m+1)^\eta} = \infty.
    \end{align}
    Thus, $\norm{q_m}^2$ must converge to zero. 
    Otherwise, $\sum_{m=0}^\infty X_m = \frac{1}{2} \beta \sum_{m=0}^\infty T_m \norm{q_m}^2$ diverges to infinity. 
    Thus, $q_m$ converges to 0 almost surely, 
    which completes the proof.
\end{proof}

\subsection{Proof of Lemma \ref{lem finite sample bounds}}
\label{sec proof lem finite sample bounds} 
\begin{proof}
    First, we prove the bounds for each function.
    
    \textbf{Bound for $g_t(\cdot)$}:
        \begin{align}
            \| g_t(w) \| &= \norm{ - \hat{A}_{t+f(t)}^\top (\hat{A}_t w + \hat{b}_t) } \le \norm{\hat{A}_{t+f(t)}^\top} (\|\hat{A}_t\| \|w\| + \|\hat{b}_t\|) \\
            &\le H (H B + H) = H^2(B+1).
        \end{align}
        
    \textbf{Bound for $\bar g(\cdot)$}:
    \begin{align}
        \norm{\bar g(w)} = \norm{A^\top (A w + b)} \le \norm{A^\top} (\norm{A} \norm{w} + \norm{b}) \le H(H B + B) = H^2(B+1).
    \end{align}
        
    \textbf{Bound for $\Lambda_t(\cdot)$}:
    \begin{align}
        |\Lambda_t(w)| &= | \langle w - w_*, g_t(w) - \bar{g}(w) \rangle| \le \|w-w_*\| \|g_t(w)-\bar{g}(w)\| \\
        &\le (\|w\| + \|w_*\|) (\|g_t(w)\| - \|\bar{g}(w)\|) \le (B + B) [H^2(B+1) + H^2(B+1)] \\
        &= 4 H^2 B(B+1).
    \end{align}

    Hence, by taking $C_g = \max\{H^2(B+1), 4 H^2 B(B+1)\}$, we have the result stated.
    Second, we prove the functions are Lipschitz.
    
    \textbf{$g_t(\cdot)$ is $H^2$-Lipschitz}: 
    \begin{align}
        \norm{g_t(w) - g_t(w')} &= \norm{-\hat A_{t+f(t)}^\top (\hat A_t w + \hat b_t) + \hat A_{t+f(t)}^\top (\hat A_t w' + \hat b_t)} = \norm{\hat A_{t+f(t)}^\top \hat A_t (w - w')} \\
        &\le \norm{\hat A_{t+f(t)}^\top} \norm{\hat A_t} \norm{w - w'} \le H^2 \norm{w - w'}
    \end{align}
    
    \textbf{$\bar g(\cdot)$ is $H^2$-Lipschitz}: 
    \begin{align}
        \norm{\bar g(w) - \bar g(w')} &= \norm{- A^\top (A w + b) + A^\top (A w' + b)} = \norm{A^\top A (w'-w)} \\
        &\le \norm{A^\top} \norm{A} \norm{w'-w} \le H^2 \norm{w'-w}.
    \end{align}
    
    \textbf{$\Lambda_t(\cdot)$ is $2H^2 (3B + 1)$-Lipschitz}: 
    \begin{align}
        &|\Lambda_t(w) - \Lambda_t(w')| \\
        =& |\langle w-w_*, g_t(w) - \bar{g}(w)\rangle - \langle w'-w_*, g_t(w') - \bar{g}(w') \rangle| \\
        =& |\langle w-w_*, g_t(w) - \bar{g}(w) - (g_t(w') - \bar{g}(w'))\rangle + \langle w-w_* - (w' - w_*), g_t(w') - \bar{g}(w') \rangle| \\
        \le& |\langle w-w_*, g_t(w) - \bar{g}(w) - (g_t(w') - \bar{g}(w'))\rangle| + |\langle w - w', g_t(w') - \bar{g}(w') \rangle| \\
        \le& \|w-w_*\| \|g_t(w) - \bar{g}(w) - (g_t(w') - \bar{g}(w'))\| + \|w - w'\| \|g_t(w') - \bar{g}(w')\| \\
        \le& (\|w\| + \|w_*\|)(\|g_t(w') - g_t(w)\| + \|\bar{g}(w) - \bar{g}(w')\|)  + (\|g_t(w')\|  + \|\bar{g}(w')\|)\|w - w'\| \\
        \le& (B + B) (\|g_t(w') - g_t(w)\| + \|\bar{g}(w) - \bar{g}(w')\|) + [H^2 (B+1)  + H^2 (B+1)]\|w - w'\| \\
        \le& 2B (\|g_t(w') - g_t(w)\| + \|\bar{g}(w) - \bar{g}(w')\|) + 2 H^2 (B+1)\|w - w'\| \\
        \le& 2B (H^2 \|w' - w\| + H^2 \|w' - w\|) + 2 H^2 (B+1)\|w - w'\| \\
        =& 2 H^2 (3B+1) \|w-w'\|.
    \end{align}
    
    Therefore, by taking $C_{Lip} = \max\{H^2, 2H^2 (3B + 1)\}$, we have the result stated.
    
    \noindent Lastly, we prove the following inequality regarding the inner product.
    \begin{align}
        \langle w - w', \bar{g}(w) - \bar{g}(w') \rangle &= \left\langle w - w', - A^\top (A w + b) + A^\top (A w' + b) \right\rangle \\
        & = - (w - w')^\top A^\top A (w - w') \le - \beta \|w-w'\|^2,
    \end{align}
    where the last inequality holds due to~\eqref{eq pd beta}. 
\end{proof}

\subsection{Proof of Lemma \ref{lem expected lambda bound}}
\label{sec proof lem expected lambda bound}
\begin{proof}
    For any $i \ge 0$, since $w_i$ lies in the ball for projection,
    we have 
    \begin{align}
        \norm{w_{i+1} - w_i} &= \norm{\Gamma(w_i + \alpha_i g_i(w_i)) - w_i} \\
        &= \norm{\Gamma(w_i + \alpha_i g_i(w_i)) - \Gamma(w_i)} \\
        &\leq \norm{\Gamma(w_i + \alpha_i g_i(w_i) - w_i)}  \explain{$\Gamma(\cdot)$ is non-expansive}\\
        &= \norm{\Gamma(\alpha_i g_i(w_i))}  \\
        &\le \norm{\alpha_i g_i(w_i)} \\
        &= \alpha_i \norm{g_i(w_i)} \\
        &\le \alpha_i C_g \explain{Lemma~\ref{lem finite sample bounds}}.
    \end{align}
    Therefore, by telescoping, we can deduce that for all $0 < k < t$, 
    \begin{align}
        \norm{w_t - w_{t-k}} \le \sum_{i=t-k}^{t-1} \norm{w_{i+1} - w_i} \le \sum_{i=t-k}^{t-1} C_g \alpha_i = C_g \sum_{i=t-k}^{t-1} \alpha_i.
    \end{align}
    Because $\frac{1}{t+1}$ is a decreasing function in $t$, for all $i\in[t, t+1]$, $\frac{1}{t+1} \le \frac{1}{i}$. 
    As a consequence,
    $\frac{1}{t+1} \le \int_{t}^{t+1} \frac{1}{i}$, and
    \begin{align}
       \sum_{i=t-k}^t \frac{1}{i+1} \le \sum_{i=t-k}^t \int_{i}^{i+1} \frac{1}{j} \le \int_{t-k}^t \frac{1}{i} = \ln(t) - \ln(t-k) = \ln\left(\frac{t}{t-k}\right). 
    \end{align}
    In addition, given that $\alpha_t = \frac{C_\alpha}{t+1}$, 
    we have
    \begin{align}
        \norm{w_t - w_{t-k}} &\le C_g \sum_{i=t-k}^{t-1} \alpha_i = C_g \sum_{i=t-k}^{t-1} \frac{C_\alpha}{i+1} \le C_g C_\alpha \ln\left(\frac{t}{t-k}\right).
    \end{align}
    Applying Lemma \ref{lem finite sample bounds} again, we get
    \begin{align}
       |\Lambda_t(w_t) - \Lambda_t(w_{t-k})| \le C_g C_\alpha C_{Lip} \ln\left(\frac{t}{t-k}\right) 
    \end{align}
    and hence
    \begin{align}
        \label{eq lambda wt in terms of lambda wt-k}
        \Lambda_t(w_t) \le \Lambda_t(w_{t-k}) + C_g C_\alpha C_{Lip} \ln\left(\frac{t}{t-k}\right).
    \end{align}

   We now bound the expectation of $\Lambda_t(w_{t-k})$ conditioning on $\fF_{t-k + f(t-k)}$.
   In particular, we have
    \begin{align}
        &\E\left[\Lambda_t(w_{t-k}) | \fF_{t-k + f(t-k)}\right] \\
        =& \E\left[\langle w_{t-k} - w_*, g_t(w_{t-k}) - \bar g(w_{t-k}) \rangle | \fF_{t-k + f(t-k)}\right] \\
        =& \E\left[\left\langle w_{t-k} - w_*, \hat A_{t+f(t)}^\top (\hat A_t w_{t-k} + \hat \hat{b}_t) - A^\top (A w_{t-k} - b) \right\rangle | \fF_{t-k + f(t-k)}\right] \\
        =& \E\left[\left\langle w_{t-k} - w_*, \left(\hat A_{t+f(t)}^\top \hat A_t - A^\top A\right) w_{t-k} - \left(\hat A_{t+f(t)}^\top \hat b_t - A^\top b\right) \right\rangle | \fF_{t-k + f(t-k)}\right]. \\
    \end{align}
    As expectation and dot product are linear and $w_{t-k}$ is adapted to $\fF_{t-k + f(t-k)}$, 
    we can further reduce our expectations as
    \begin{align}
        &\E\left[\Lambda_t(w_{t-k}) | \fF_{t-k + f(t-k)}\right] \\
        =&\left\langle w_{t-k} - w_*, \left(\E\left[\hat A_{t+f(t)}^\top \hat A_t| \fF_{t-k + f(t-k)}\right]  - A^\top A\right) w_{t-k} - \left(\E\left[\hat A_{t+f(t)}^\top \hat b_t  | \fF_{t-k + f(t-k)}\right] - A^\top b\right) \right\rangle \\
        \le& \norm{w_{t-k} - w_*} \norm{\left(\E\left[\hat A_{t+f(t)}^\top \hat A_t| \fF_{t-k + f(t-k)}\right]  - A^\top A\right) w_{t-k} - \left(\E\left[\hat A_{t+f(t)}^\top \hat b_t  | \fF_{t-k + f(t-k)}\right] - A^\top b\right)} \\
        \le& (\norm{w_{t-k}} + \norm{w_*}) \left(\norm{\left(\E\left[\hat A_{t+f(t)}^\top \hat A_t| \fF_{t-k + f(t-k)}\right]  - A^\top A\right) w_{t-k}} + \norm{\E\left[\hat A_{t+f(t)}^\top \hat b_t  | \fF_{t-k + f(t-k)}\right] - A^\top b}\right) \\
        \le& (\norm{w_{t-k}} + \norm{w_*}) \left(\norm{\E\left[\hat A_{t+f(t)}^\top \hat A_t| \fF_{t-k + f(t-k)}\right]  - A^\top A} \norm{w_{t-k}} + \norm{\E\left[\hat A_{t+f(t)}^\top \hat b_t  | \fF_{t-k + f(t-k)}\right] - A^\top b}\right) \\
        \le& (B + B) \left(\norm{\E\left[\hat A_{t+f(t)}^\top \hat A_t| \fF_{t-k + f(t-k)}\right]  - A^\top A} B + \norm{\E\left[\hat A_{t+f(t)}^\top \hat b_t  | \fF_{t-k + f(t-k)}\right] - A^\top b}\right) \\
        \le& 2B\left(\norm{\E\left[\hat A_{t+f(t)}^\top \hat A_t| \fF_{t-k + f(t-k)}\right]  - A^\top A} B + \norm{\E\left[\hat A_{t+f(t)}^\top \hat b_t  | \fF_{t-k + f(t-k)}\right] - A^\top b}\right).
    \end{align}
    Applying Lemma \ref{lem stochastic estimates}, 
    we get
    \begin{align}
        &\E[\Lambda_t(w_{t-k}) | \fF_{t-k + f(t-k)}]
        \le& 2B(B+1) C_M \begin{cases}
            1 & t < t-k + f(t-k)\\
            \chi^{f(t)} + \chi^{t - (t-k + f(t-k))} & t-k + f(t-k) \le t\\
        \end{cases}.
    \end{align}
    Taking total expectations then yields 
    \begin{align}
      \label{eq expectation of lambda wt-k}
      \E[\Lambda_t(w_{t-k})] \le 2B(B+1) C_M \begin{cases}
            1 & k < f(t-k)\\
            \chi^{f(t)} + \chi^{t - (t-k + f(t-k))} & k \ge f(t-k)\\
        \end{cases}.
    \end{align}
    Plugging \eqref{eq expectation of lambda wt-k} into the expectation of \eqref{eq lambda wt in terms of lambda wt-k} yields
    \begin{align}
        &\E[\Lambda_t(w_t)]
        \le& C_g C_\alpha C_{Lip} \ln\left(\frac{t}{t-k}\right) + 2B(B+1) C_M \begin{cases}
            1 & k < f(t-k)\\
            \chi^{f(t)} + \chi^{t - (t-k + f(t-k))} & k \ge f(t-k)
            \end{cases},
    \end{align}
    which completes the proof.
\end{proof}

\end{document}